\documentclass[twoside,11pt]{article}

\usepackage{microtype}
\usepackage{graphicx}
\usepackage{subfigure}
\usepackage{booktabs} 
\usepackage{amsmath}
\usepackage{amssymb}

\usepackage{algorithm}
\usepackage[algo2e,ruled,linesnumbered,lined]{algorithm2e}

\usepackage{bm}
\newtheorem{assumption}{Assumption}
\usepackage[font=small,labelfont=bf]{caption}

\usepackage{jmlr2e}
\usepackage{xcolor}
\usepackage{color}

\usepackage[normalem]{ulem}

\usepackage{caption}
\usepackage{graphicx}

\begin{document}

\title{Decoupled Greedy Learning of CNNs for Synchronous and Asynchronous Distributed Learning}

\author{\name Eugene Belilovsky \email eugene.belilovsky@concordia.ca \\
       \addr Concordia University and Mila\\
       Montreal, Canada\\
       \AND
        \name Louis Leconte \email louis.leconte@ens-paris-saclay.fr \\
       \addr LIP6, Sorbonne University\\
       CMAP, Ecole Polytechnique, France\\
       \AND
       \name Lucas Caccia \email lucas.page-caccia@mail.mcgill.ca \\
       \addr 
        McGill University and Mila\\
        Montreal, Canada\\
       \AND
       \name Michael Eickenberg \email meickenberg@flatironinstitute.org \\
       \addr CCM, Flatiron Institute\\
       \AND
       \name Edouard Oyallon \email edouard.oyallon@lip6.fr \\
       \addr CNRS, LIP6, Sorbonne University\\
       Paris, France
       }

\editor{}

\maketitle{}


\begin{abstract}
A commonly cited inefficiency of neural network training using back-propagation is the \textit{update locking} problem: each layer must wait for the signal to propagate through the full network before updating. Several alternatives that can alleviate this issue have been proposed. 
In this context, we consider a simple alternative based on minimal feedback,
which we call \emph{Decoupled Greedy Learning (DGL)}. It is based on a classic greedy relaxation of the joint training objective, recently shown to be effective in the context of Convolutional Neural Networks (CNNs) on large-scale image classification. We consider an optimization of this objective that permits us to decouple the layer training, allowing for layers or modules in networks to be trained with a potentially linear parallelization. With the use of a replay buffer we show that this approach can be extended to asynchronous settings, where modules can operate and continue to update with possibly large communication delays. To address bandwidth and memory issues we propose an approach based on online vector quantization. This allows to drastically reduce the communication bandwidth between modules and  required memory for replay buffers. 
 We show theoretically and empirically that this approach converges and compare it to the sequential solvers. 
We demonstrate the effectiveness of DGL against alternative approaches  on the CIFAR-10 dataset and on the large-scale ImageNet dataset. 
\end{abstract}

\begin{figure*}
    \centering
    \includegraphics[width=1.0\linewidth]{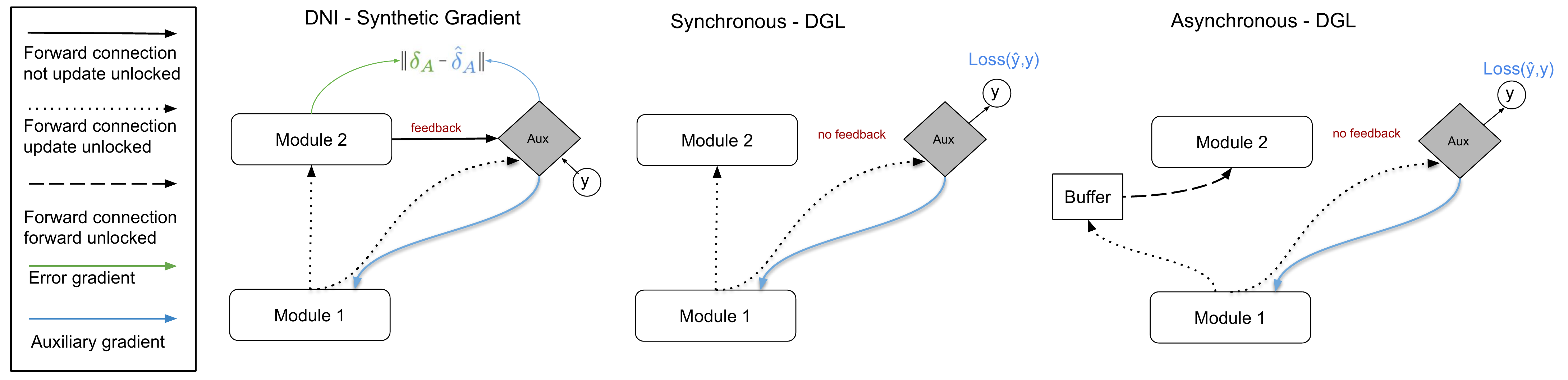}
    \caption{Comparison of DNI, Synchronous and Asynchronous DGL. Note that in DGL subsequent modules do not provide feedback to earlier modules, thus removing dependencies of the auxiliary network on backpropagated gradient information. Asynchronous DGL allows achieving forward unlocking (updates can be performed without waiting for prior modules).}
    \label{fig:dni_dgl}
\end{figure*}

\section{Introduction}
Jointly training all layers using back-propagation and stochastic optimization is the standard method for learning neural networks, including the computationally intensive Convolutional Neural Networks (CNNs). 
Due to the sequential nature of gradient processing, standard back-propagation has several well-known inefficiencies that prohibit parallelization of the computations of the different constituent modules. \citet{jaderberg2016decoupled} characterize these inefficiencies in order of severity as the 
\textit{forward-}, 
\textit{update-}, and 
\textit{backward} 
\textit{locking} problems. 
Backward 
\textit{\textbf{un}locking}
of a module
would permit updates of all modules once 
forward
signals have propagated to all subsequent modules, update 
\textit{\textbf{un}locking}
would permit updates of a module \textit{before} a signal has reached all subsequent modules, and forward 
\textit{\textbf{un}locking}
would permit a module to operate asynchronously from its predecessor and dependent modules. 

Methods 
addressing backward locking to a certain degree have been proposed in \citep{ddg,Huo2018,choromanska2018beyond,direct_feedback}. However, update locking is a far more severe inefficiency. Thus \citet{jaderberg2016decoupled} and \citet{czarnecki2017} propose and analyze Decoupled Neural Interfaces (DNI), a method that
uses an auxiliary network to predict the gradient of the backward pass directly from the input. This method
unfortunately does not scale well computationally or in terms of accuracy, especially in the case of CNNs \citep{Huo2018,ddg}. Indeed, auxiliary networks must predict a  weight gradient, usually high-dimensional for larger models and input sizes.

A major 
obstacle to
update unlocking 
is the heavy reliance on
the upper modules for feedback.
Several works have recently revisited the classic \citep{ivakhnenko,bengio2007greedy} approach of supervised greedy layerwise training of neural networks \citep{huang2017learning,Marquez}. In \citet{shallow}  it is shown that such an approach, which relaxes the joint learning objective, and does not require global feedback, can lead to high-performance deep CNNs on large-scale datasets. 
We will show that the greedy sequential learning objective used in these papers can be efficiently solved with an alternative parallel optimization algorithm, which permits decoupling the computations and 
achieves update unlocking. This then opens the door to extend to a forward unlocked model, which is a challenge not effectively addressed by any of the prior work. 
In particular, we use replay buffers \citep{lin1992self} to achieve
forward unlocking. This simple strategy can be shown to be a superior baseline for parallelizing the training across modules of a neural network. A training procedure that permits forward and update unlocked computation allows for model components to communicate infrequently and be trained in low-bandwidth settings such as across geographically distant nodes.

{The present paper is an extended version of \citep{belilovsky2020decoupled} that expands the asynchronous (forward unlocked) algorithm, addressing some of its key limitations. In particular, a major issue of replay buffers used to store intermediate representations is that they can grow large, 
requiring large amounts of
on-node memory and more critically increasing inter-node bandwidth. We introduce and evaluate a potential path for drastically reducing the bandwidth constraints between nodes in the case of  asynchronous DGL as well as reducing the need for on-node memory. 
In order to address this issue, a natural approach is compression. However, the replay buffers here store \textit{time-varying} activations. In this context and inspired by \cite{caccia2019online,oord2017neural} we thus propose a computationally efficient method for online learned compression using codebooks. Specifically, we introduce in Sec. \ref{subsec:qt} and test in Sec. \ref{sec:xp_qt}, the use of a quantization module that compresses the intermediary features used between successive layers but is able to rapidly adapt to distribution shifts. In complement to \cite{belilovsky2020decoupled}, this module can deal with online distributions and regularly update a code-book that memorizes some attributes of the current stream of data. It allows to both reduce significantly the local memory of a
node as well as the transmission between two subsequent machines, without significantly decreasing the final accuracy of our models. We further show that for a fixed budget of memory, this new method outperforms the algorithm introduced in \cite{belilovsky2020decoupled}.}  

 The paper is structured as follows. In Sec. 2 we propose an optimization procedure for a decoupled greedy learning objective that achieves \textit{update unlocking} and then extend it to an asynchronous setting (async-DGL) using a replay buffer, addressing \textit{forward unlocking}. {Further, we introduce a new quantization module that reduces the memory use of our method}. In Sec. 3 we show that the proposed optimization procedure converges and recovers standard rates of non-convex optimization, motivating empirical observations in the subsequent experimental section. In Sec. 4 we  show that DGL can outperform competing methods in terms of scalability to larger and deeper models and stability to optimization hyperparameters and overall parallelism, allowing it to be applied to large datasets such as ImageNet. {In Sec. \ref{sec:xp_qt}, we test our new quantization module on 
 CIFAR-10.} We extensively study async-DGL and find that it is robust to significant delays. {In several experiments we show that buffer quantization improves both performance and memory consumption.} We also empirically study the impact of parallelized training on convergence. Code for experiments is included in the submission.

\section{Parallel Decoupled Greedy Learning}

In this section we 
formally define the greedy objective and parallel optimization which we 
study in both the synchronous and asynchronous setting. We 
mainly consider the online setting and assume a stream of samples or mini-batches denoted $\mathcal{S}\triangleq\{(x_0^t,y^t)\}_{t\leq T}$, run during $T$ iterations.
\subsection{Preliminaries}
For comparison purposes, we briefly review the update unlocking approach from  DNI \citep{jaderberg2016decoupled}. There, each network module has an associated \textit{auxiliary net} which, given the output activation of the module, predicts the gradient signal from subsequent modules: the module can thus perform an update while modules above are still forward processing. The DNI auxiliary model is trained by using true gradients provided by upper modules when they become available, requiring activation caching. This also means that the auxiliary module can become out of sync with the changing output activation distribution, often requiring slow learning rates. Due to this and the high dimensionality of the predicted gradient which scales with module size, this estimate is challenging. One may ask how well a method that entirely avoids the use of feedback from upper modules would fare given similarly-sized auxiliary networks. We will show that adapting the objective in \citep{shallow,bengio2007greedy} can also allow for update unlock and a degree of forward unlocking, with better properties.

\setlength{\textfloatsep}{1pt}
    \begin{algorithm2e}\small
    \caption{Synchronous DGL}\label{algo:basic}
  \SetAlgoLined
  \DontPrintSemicolon
\KwIn{Stream $\mathcal{S}\triangleq\{(x_0^t,y^t)\}_{t\leq T}$ of samples or mini-batches.}
\textbf{Initialize} Parameters $\{\theta_j,\gamma_j\}_{j\leq J}$.\;
\For { $(x_0^t,y^t) \in \mathcal{S}$ }
{
\For {$j \in 1,..., J$}
   {$x^t_j \leftarrow f_{\theta_{j-1}}(x^t_{j-1})$.\;
Compute $\nabla_{(\gamma_j,\theta_j)}\hat{ \mathcal{L}}(y^t,x^t_j;\gamma_j,\theta_j)$.\;
$(\theta_j,\gamma_j)\leftarrow$Update parameters $(\theta_j,\gamma_j)$.
   }
 }
\end{algorithm2e}\begin{algorithm2e}\small
 \SetAlgoLined
  \DontPrintSemicolon
    \KwIn{Stream $\mathcal{S}\triangleq\{(x_0^t,y^t)\}_{t\leq T}$;  Distribution of the delay $p=\{p(j)\}_{j}$; Buffer size $M$.}
 \textbf{Initialize:} Buffers $\{B_j\}_{j}$; params $\{\theta_j,\gamma_j\}_{j}$.\\
 \While{\normalfont{\textbf{ training}}}{Sample $j$ in $\{1,...,J\}$ following $p$.\\
   \uIf{$j=1$}{ $ (x_{0},y)\gets \mathcal{S}$}\Else{ $(x_{j-1},y)\gets B_{j-1}$.}
   $x_j \leftarrow f_{\theta_{j-1}}(x_{j-1})$.\;
    Compute $\nabla_{(\gamma_j,\theta_j)}\hat{ \mathcal{L}}(y,x_j;\gamma_j,\theta_j)$.\;
     $(\theta_j,\gamma_j)\leftarrow$ Update parameters $(\theta_j,\gamma_j)$.\;
    \lIf{$j<J$}{
    $B_{j} \gets (x^{j},y)$.
           }}
    \caption{Asynchronous DGL with Replay\label{algo:buffer_sym}}\end{algorithm2e}

\begin{algorithm2e}\small
 \SetAlgoLined
  \DontPrintSemicolon
    \KwIn{Stream $\mathcal{S}\triangleq\{(x_0^t,y^t)\}_{t\leq T}$;  Distribution of the delay $p=\{p(j)\}_{j}$; Buffer size $M$. Codebooks $\{\mathcal{C}_j\}_j$. Codebook update delay $T_{code}$}
 \textbf{Initialize:} Buffers $\{B_j\}_{j}$; params $\{\theta_j,\gamma_j\}_{j}$; codebooks $\{\mathcal{C_j}\}_j$.\\
 \While{\normalfont{\textbf{ training}}}{Sample $j$ in $\{1,...,J\}$ following $p$.\\
   \uIf{$j=1$}{ $ (x_{0},y)\gets \mathcal{S}$}\Else{ $(\tilde x_{j-1},y)\gets B_{j-1}$.\;
   $x_{j-1}=\text{UnQuantize}(\tilde x_{j-1},\mathcal{C}_j)$
   }
   $x_j \leftarrow f_{\theta_{j-1}}(x_{j-1})$.\;
    Compute $\nabla_{(\gamma_j,\theta_j)}\hat{ \mathcal{L}}(y,x_j;\gamma_j,\theta_j)$.\;
     $(\theta_j,\gamma_j)\leftarrow$ Update parameters $(\theta_j,\gamma_j)$.\;
      If $j\%T_{code}$: \textsc{ReceiveCodebooks} $\{\mathcal{C}_j\}$.\; 
    \lIf{$j<J$}{
    $B_{j} \gets (\text{Quantize}[x^{j}],y)$.
           }
      
           }
           
    \caption{Asynchronous DGL with Replay and Quantized modules \label{algo:buffer_sym_quantized}}\end{algorithm2e}
  



\subsection{Optimization for Greedy Objective} \label{sec:main}

Let $\bm{X}_0$ and $Y$  be the data and labels, $\bm{X_j}$ be the output representation for module $j$. We will denote the per-module objective function $\hat{\mathcal{L}}(\bm{X}_{j}, Y;\theta_{j},\gamma_{j})$, where the parameters $\theta_j$ correspond to the module parameter (i.e. $\bm{X}_{j+1} =f_{\theta_{j}}(\bm{X}_{j})$). Here $\gamma_j$ represents parameters of the auxiliary networks used to predict the final target and compute the local objective. $\hat{\mathcal{L}}$ in our case will be the empirical risk with a cross-entropy loss. The greedy training objective is thus given recursively by defining $P_j$:
\begin{equation}
\min_{\theta_j,\gamma_j}\hat{\mathcal{L}}(\bm{X}_{j}, Y;\theta_{j},\gamma_{j})\tag{$P_j$}\label{eq:dap},
\end{equation}
where $\bm{X}_{j} =f_{\theta_{j-1}^*}(\bm{X}_{j-1})$ and $\theta_{j-1}^*$ is the minimizer of Problem  
(\(P_{j-1}\)). 
A natural way to solve the optimization problem for $J$ modules, $(P_J)$, is thus by sequentially solving the problems $\{P_{j}\}_{j\leq J}$ starting with $j=1$. This is the approach taken in e.g. \citet{Marquez,huang2017learning,bengio2007greedy,shallow}. 
Here we consider an alternative procedure for optimizing the same objective, which we refer to as Sync-DGL. It is outlined in Alg \ref{algo:basic}. In Sync-DGL individual updates of each set of parameters are performed in parallel across the different layers. Each layer processes a sample or mini-batch, then passes it to the next layer, while simultaneously performing an update based on its own local loss. Note that at line $5$ the subsequent layer can already begin computing line $4$. Therefore, this algorithm achieves update unlocking. Once $x_j^t$ has been computed,   subsequent layers can begin processing. Sync-DGL can also be seen as a generalization of the biologically plausible learning method proposed in concurrent work \citep{nokland2019training}.
Appendix D  also gives an explicit version of an equivalent multi-worker pseudo-code. Fig.~\ref{fig:dni_dgl} illustrates the decoupling compared to how samples are processed in the DNI algorithm. 

 In this work we solve the sub-problems $P_j$ by backpropagation, but we note that any iterative solver available for $P_j$ will be applicable (e.g. \cite{choromanska2018beyond}).
 Finally we emphasize that unlike the sequential solvers of (e.g. \citet{bengio2007greedy,shallow}) the distribution of inputs to each sub-problem solver changes over time, resulting in a learning dynamic whose properties have never been studied nor contrasted with sequential solvers.

\subsection{Asynchronous DGL with Replay} \label{sec:asynch}

We can now extend this framework to address
\textit{forward unlocking} \citep{jaderberg2016decoupled}. DGL 
modules already do not depend on their successors for updates. 
We can further reduce dependency on the previous modules such that they can operate asynchronously. 
This is achieved via a replay buffer that is shared between adjacent modules, enabling them to reuse older samples. 
Scenarios with communication delays or substantial variations in speed between layers/modules
benefit from this. 
We study one instance of such an algorithm that uses a replay buffer of size $M$, shown in Alg. \ref{algo:buffer_sym} and illustrated in Fig.~\ref{fig:dni_dgl}. 





Our minimal distributed setting is as follows. Each worker $j$ has a buffer that it writes to and that worker $j+1$ can read from. The buffer uses a simple read/write protocol. A buffer $B_j$ lets layer $j$ write new samples. When it reaches capacity it overwrites the oldest sample. Layer $j+1$ requests samples from the buffer $B_j$. They are selected by a last-in-first-out (LIFO) rule, with precedence for the least reused samples. 
Alg. \ref{algo:buffer_sym} simulates potential delays in such a setup by the use of a probability mass function (pmf) $p(j)$ over workers, analogous to typical asynchronous settings such as \citep{leblond2017asaga}. At each iteration, a layer is chosen at random according to $p(j)$ to perform a computation. In our experiments we  limit ourselves to pmfs that are uniform over workers except for a single layer which is chosen to be selected less frequently on average. Even in the case of a uniform pmf, asynchronous behavior will naturally arise, requiring the reuse of samples.  Alg.~\ref{algo:buffer_sym} permits a controlled simulation of processing speed discrepancies and will be used over  settings of $p$ and $M$ to demonstrate that training and testing accuracy remain robust in practical regimes. Appendix D  also provides 
pseudo-code for implementation in a parallel environment.

Unlike common data-parallel asynchronous algorithms \citep{elasticSGD}, the asynchronous DGL does not rely on a master node and requires only local communication similar to recent decentralized schemes \citep{lian2017asynchronous}. Contrary to decentralized SGD, DGL nodes only need to maintain and update the parameters of their local module, permitting much larger modules.  
Combining asynchronous DGL with distributed synchronous SGD for sub-problem optimization is a promising direction. For example it can alleviate a common issue of the popular distributed synchronous SGD in deep CNNs, which is the often limiting maximum batch size \citep{im1hr}. 

\subsection{Reducing the memory- and communication footprint of Asynchronous DGL}\label{subsec:qt}

\begin{figure*}
    \centering
    \includegraphics[width=1.0\linewidth]{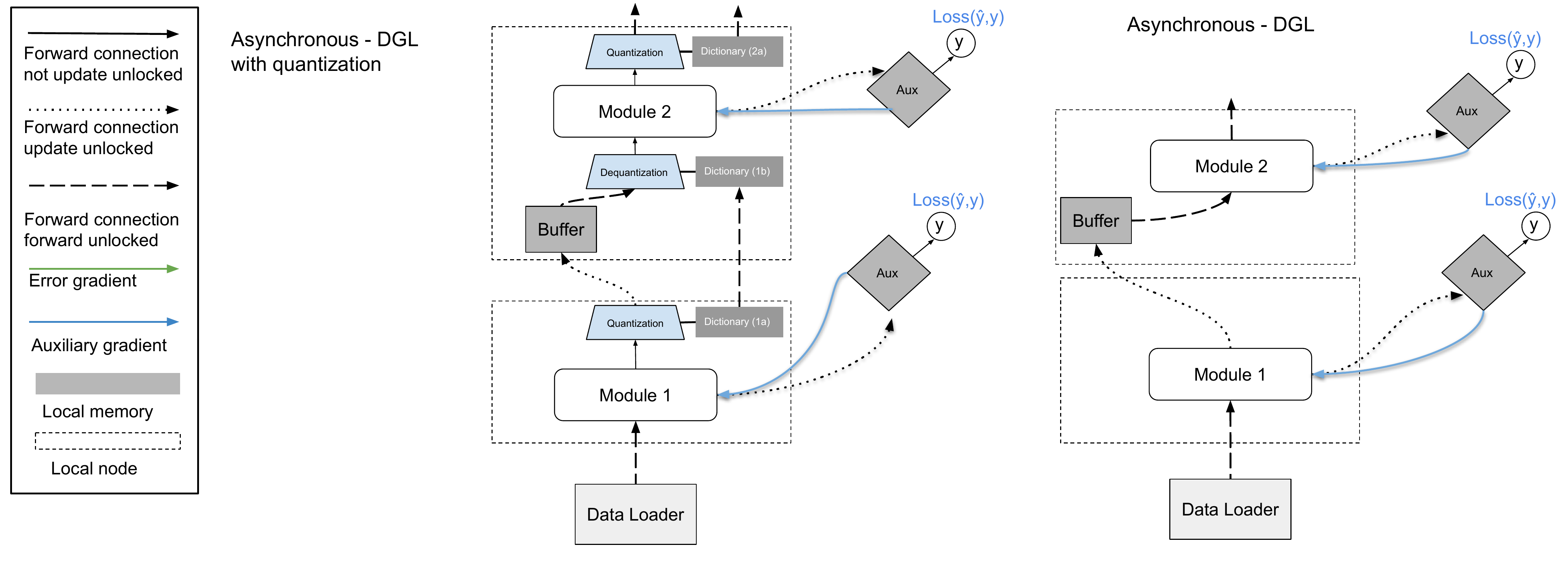}
    \caption{Schematic diagram that explains how the buffers and encoder/decoder quantization modules are incorporated. (left) with quantization (right) without quantization.}
    \label{fig:dni_dgl_QT}
\end{figure*}

\if False
The typical objective of an  optimization procedure distributed along $H$ nodes is to linearly reduce by   $\mathcal{O}(H)$ the training time compared to using 
a single node. However, in practical applications, the speed-up is in general sub-linear 
due to
communication issues: 
bandwidth restrictions (e.g., a bandwidth bottleneck or a significant lag between distant nodes) can substantially increase the communication time.
For instance, in a Federated Learning setting \cite{philippenko2020artemis}, a master node will have to send $H$ copies of its parameters to the slave nodes, and if the communication time is not significantly smaller than the local computation time of each of the $H$ nodes, no speed-up can be observed. Furthermore,  decentralized asynchronous applications like ours can require 
storing
intermediate
results and this can significantly increase the local memory footprint needed by a node.

\fi 

The typical objective of an  optimization procedure distributed along $H$ nodes is to linearly reduce by   $\mathcal{O}(H)$ the training time compared to using 
a single node. However, in practical applications, the speed-up is in general sub-linear due to communication issues: bandwidth restrictions (e.g., a bandwidth bottleneck or a significant lag between distant nodes) can substantially increase the communication time. In a distributed use case Sync-DGL and Async-DGL would be bottlenecked by communication bandwidth. As the model grows larger, features must be sent across nodes. Similarly for Async DGL the local memory footprint might also be an issue if each node is a device with limited available computational resources. We propose a solution to this problem based on an online dictionary learning algorithm, which is able to rapidly adapt a compression algorithm to the changing features at each node, leading to a large reduction in communication bandwidth and memory. The method is illustrated in Fig. \ref{fig:dni_dgl_QT} and the algorithm we used is given in Alg. \ref{algo:buffer_sym_quantized}.

As illustrated in the Figure \ref{fig:dni_dgl_QT} we propose to incorporate a quantization module that relies on a codebook with $C$ atoms. 
Following \cite{van2017neural}, the quantization step works as follows: each output feature is assigned to its closest atom in its local encoding codebook and the decoding step consists 
simply in recovering the corresponding atom via its local decoding codebook. The  numbers of bits required to encode a single 
quantized vector is thus $\lceil\log_2(C)\rceil$ bits.


During training, the distribution of features at each layer is changing, so the codebooks must be updated online. We must also send the codes to the subsequent node to synchronize the codebooks of two distant communicating modules, so that the following node can decode.  Notably the rate at which codebooks are synchronized need not be the same as the rate at which features are sent. We write $\alpha\in [0,1]$ the synchronization rate of the codebooks: We only synchronize the codebook during a selected fraction $\alpha$ of the training iterations. Empirically we will illustrate in the sequel that the codebooks can be synchronized infrequently as compared to the rate a module sends out features. The codebook is updated via an EM algorithm  that is learning to minimize the reconstruction error of a given batch of samples, as done in 
\cite{van2017neural, caccia2019online}. 
In order to deal with batches of data, the dictionary is updated with Exponential Moving Averages (EMA).

We will now discuss the bandwidth and memory savings of the quantization module by deriving explicit equations. First, let us introduce the necessary notations. In the following, at a given module indexed by $j$, we write $\mathcal{B}$ the batch size of a batch $x_j$ of features with dimension $K_j \times N_j^2$, where $N_j$ corresponds to the spatial size of $x_j$ and $K_j$ is the number of channels. Assuming the variables are coded 
as 32-bit floating point numbers%
, this implies that without quantization, a batch of features will require $32\mathcal{B}N_j^2K_j\,\text{bits}$. Similarly, the buffer memory, which is required in this case, corresponds to $32MN_j^2K_j$, where we always have $M>\mathcal{B}$ in order to avoid sampling issues for sampling a new batch of data.

As in \cite{caccia2019online, van2017neural} we incorporate a spatial structure to our encoding-decoding mechanism: our quantization algorithm will encode
the feature vector at
each spatial location
using the same encoding procedure,
leading to a spatial array of codebook indices.
Furthermore, as done in \cite{caccia2019online}, we assume that we use $k=32$ codebooks to encode respectively each fraction $\lfloor \frac{K_j}{k}\rfloor$ of the channel of a given batch. This implies that the communication between two successive modules will require for a single batch of size $\mathcal{B}$:
\begin{equation}
    \mathcal{B}kN_j^2\lceil\log_2(C)\rceil+\alpha 32 (K_j+K_{j-1})C\,\text{ bits.}
\end{equation}
Obtained in a similar fashion, the memory footprint of the buffer will be reduced to:
\begin{equation}
   MkN_j^2\lceil\log_2(C)\rceil+ 32K_jC\,\text{ bits}\,.
\end{equation}

\begin{figure*}[t]
    \centering
    \includegraphics[scale=0.55]{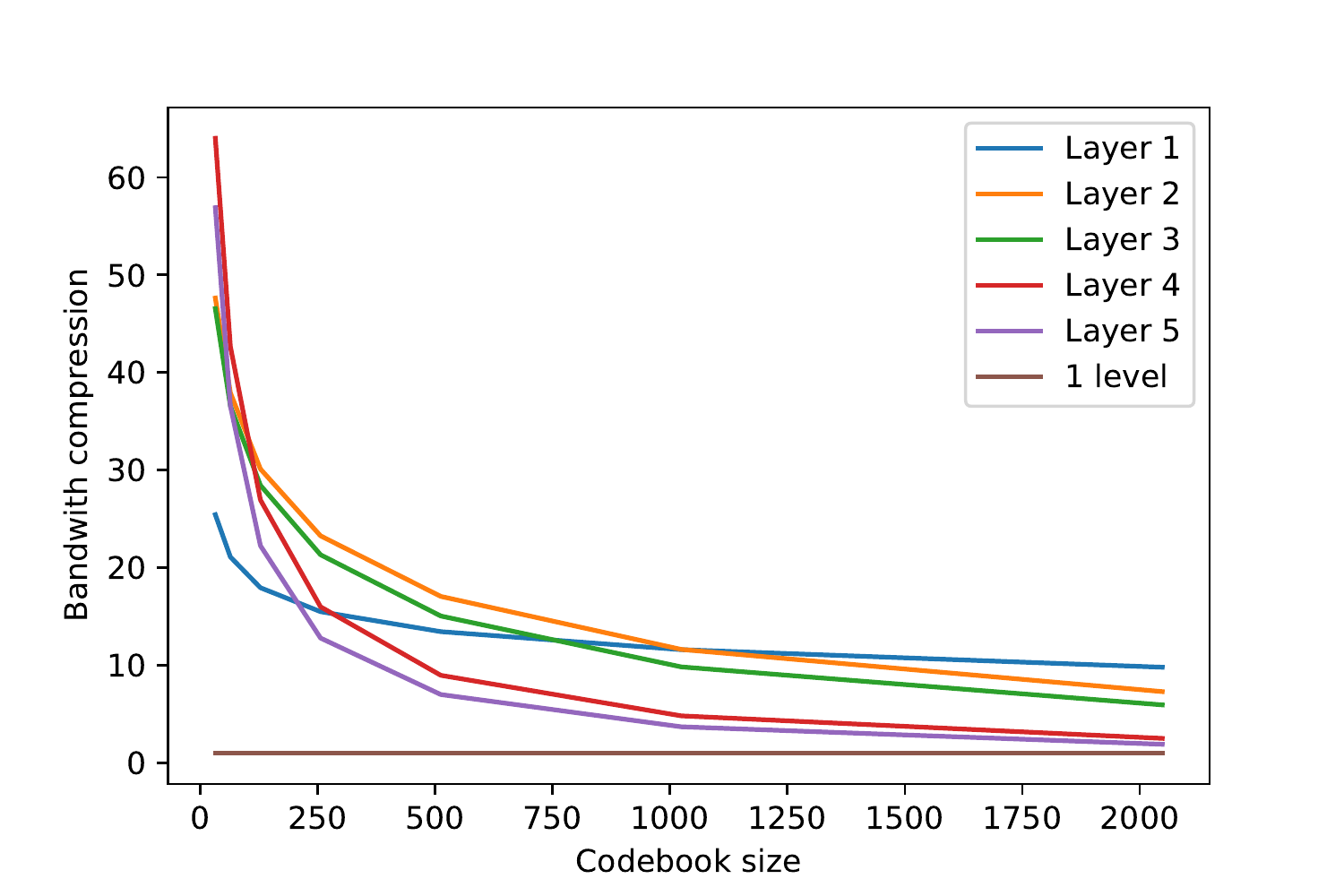}\includegraphics[scale=0.55]{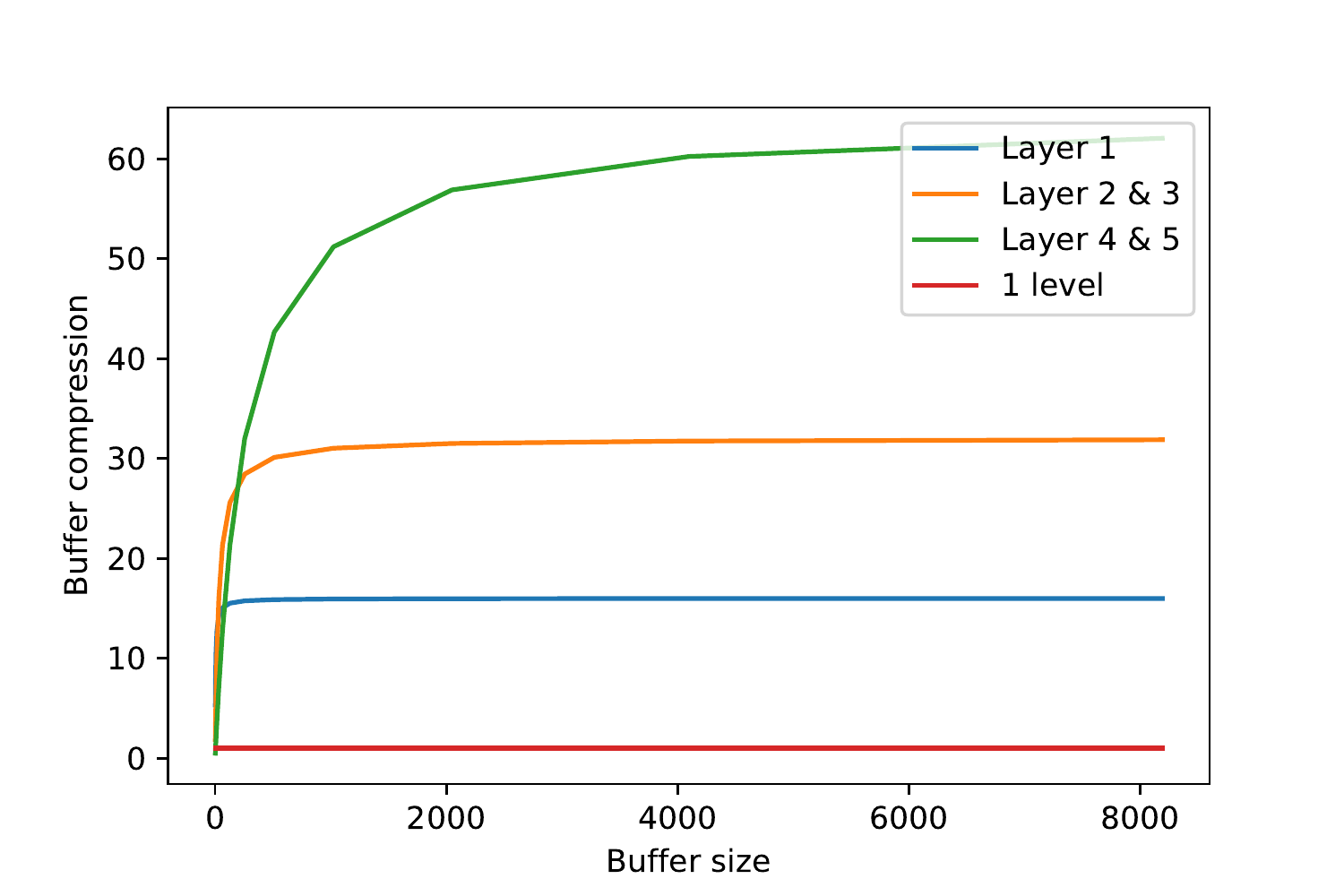}
    \caption{Illustration of bandwidth and buffer memory Compression factors as given by Eq. ~\eqref{eq:bandwith} and Eq. ~\eqref{eq:buffer} for $\alpha=1$ at various depths for a specific CNN model (from Sec. \ref{sec:xp_qt}). On the left figure, we display the bandwidth compression factor as a function of the number of codebook atoms. The right figure corresponds to the buffer compression factor when the buffer size increases. We observe that in all cases we have substantial reduction in bandwidth and memory. Note the buffer size memory compression is variable due to the size of the codebook. For simplicity, we employed the architecture of Sec. \ref{replay}. }
    \label{fig:compression_factors}
\end{figure*}

As a consequence, we can define two different compression factors when implementing the quantization inside the buffer. The bandwidth compression is defined as 
\begin{equation}\label{eq:bandwith}
C_b  = \frac{32\mathcal{B}N_j^2K_j}{\mathcal{B}kN_j^2\log_2(C)+\alpha 32 (K_j+K_{j-1})C},
\end{equation}
and indicates an improvement when it is greater than 1.
The buffer compression is defined as
\begin{equation} \label{eq:buffer}
C_n  = \frac{32MN_j^2K_j}{MkN_j^2\log_2(C)+ 32K_jC},
\end{equation}
and indicates an improvement if it is greater than 1.

As an illustration in Fig.~\ref{fig:compression_factors}  displays the bandwidth reduction factor for a reference network, similar to one to be studied in the sequel, as a function of codebook size $C$.  Our reference network has $N_j\in [32, 16,16, 8,8]$, $B=128$, $K_j\in [128, 256, 256, 512, 512]$, and $\alpha=1$, as well as the buffer memory reduction factor as a function of the buffer size $M$ at a constant codebook size of $C=256$. Note that in both cases, the memory footprint used by the codebook is small compared that of the quantized features. The compression at the buffer level increases when the buffer does. It reaches a threshold defined by the ratio of codebook encoding size and uncompressed channel sizes. However the Bandwidth compression decreases, as increasing codebook size becomes more and more similar to working in 32-bit precision.

\subsection{Auxiliary and Primary Network Design}\label{subsec:aux}

Like DNI our procedure relies on an auxiliary network to obtain update signal. Both methods thus require auxiliary network design in addition to the main CNN architecture. \citet{shallow} have shown that spatial averaging operations can be used to construct a scalable auxiliary network for the same objective as used in Sec~\ref{sec:main}. However, they did not directly consider the parallel training use case, where additional care must be taken in the design: The primary consideration is the relative speed of the auxiliary network with respect to its associated main network module. We will use primarily FLOP count in our analysis and aim to restrict our auxiliary networks to be $5\%$ of the main network. 

Although auxiliary network design might seem like an additional layer of complexity in CNN design and may require invoking slightly different architecture principles, this is not inherently prohibitive since architecture design is often related to training (e.g., the use of residuals  is originally motivated by optimization issues inherent to end-to-end backprop \citep{he2016deep}). 

Finally, we note that although we focus on the distributed learning context, this algorithm and associated theory for greedy objectives is generic and has other potential applications. For example greedy objectives have recently been used in \citep{haarnoja18a,huang2017learning} and even with a single worker DGL reduces memory.

\section{Theoretical Analysis}
We now study the convergence results of DGL. Since we do not rely on any approximated gradients, we can derive stronger properties than DNI \cite{czarnecki2017}, such as a rate of convergence in our non-convex setting.
To do so, we analyze  Alg. \ref{algo:basic} when the update steps are obtained from stochastic gradient methods. We show convergence guarantees \citep{bottou2018optimization} under reasonable assumptions. In standard stochastic optimization schemes, the input distribution fed to a model is fixed \citep{bottou2018optimization}. In this work, the input distribution to each module is time-varying and dependent on the convergence state
of the previous module. At time step $t$, for simplicity we will denote all parameters of a module (including auxiliary) as $\Theta_j^t\triangleq(\theta_j^t,\gamma_j^t)$, and samples as $Z_j^t\triangleq (X_j^t,Y^t)$, which follow the density $p_j^t(z)$.  For each auxiliary problem, we aim to prove the strongest existing guarantees \citep{bottou2018optimization,Huo2018} for the  non-convex setting despite time-varying input distributions from prior modules. Proofs 
are given in the Appendix.

Let us fix a depth $j$, such that $j>1$ and consider the converged density of the previous layer, $p^*_{j-1}(z)$. We  consider the total variation distance: $c^t_{j-1}\triangleq \int |p_{j-1}^t(z)-p_{j-1}^*(z)|\,dz$. Denoting $\ell$ the composition of the non-negative loss function and the network, we will study the expected risk $\mathcal{L}(\Theta_{j})\triangleq \mathbb{E}_{p^*_{j-1}}[\ell(Z_{j-1};\Theta_j)]$. We will now state several standard assumptions we use.
\begin{assumption}[$L$-smoothness] $\mathcal{L}$ is differentiable and its gradient is  $L$-Lipschitz.\end{assumption}
We consider the SGD scheme with learning rate $\{\eta_t\}_t$:
\begin{equation}
\hspace{-0.5em}\Theta^{t+1}_{j}=\Theta^t_j-\eta_t \nabla_{\Theta_j} \ell(Z_{j-1}^t;\Theta_j^t),
\end{equation}
where \(Z^t_{j-1}\sim p_{j-1}^t\).

 \begin{assumption}[Robbins-Monro conditions] The step sizes satisfy $\sum_t\eta_t=\infty$ yet $\sum_t\eta_t^2<\infty$.\end{assumption}
 
\begin{assumption}[Finite variance] There exists $G>0$ such that $\forall t,\Theta_j, \mathbb{E}_{p^t_{j-1}}\big[\Vert\nabla_{\Theta_j}\ell(Z_{j-1};\Theta_j)\Vert^2\big]\leq G$.\end{assumption}

The Assumptions 1, 2 and 3 are standard \citep{bottou2018optimization, Huo2018}, and we show in the following  that our proof of convergence leads to similar rates, up to a multiplicative constant. The following assumption is specific to our setting where we consider a time-varying distribution:
\begin{assumption}[Convergence of the previous layer] We assume that $\sum_{t} c^t_{j-1}<\infty$.\end{assumption}
\begin{lemma}Under Assumption 3 and 4, for all $\Theta_j,$ one has $ \mathbb{E}_{p^*_{j-1}}\big[\Vert\nabla_{\Theta_j}\ell(Z_{j-1};\Theta_j)\Vert^2\big]\leq G$.\end{lemma}



We are now ready to prove the core statement for the convergence results in this setting:
\begin{lemma}\label{lemma:main}Under Assumptions 1, 3 and 4, we have:
\begin{eqnarray*}
\mathbb{E}[\mathcal{L}(\Theta_j^{t+1})] & \leq & \mathbb{E}[\mathcal{L}(\Theta_j^{t})]+\frac{LG}{2}\eta_t^2\,\\
{} & {} & -\eta_t\big(\mathbb{E}[\Vert\nabla\mathcal{L}(\Theta_j^t)\Vert^2] -G\sqrt{2c^t_{j-1}}\big).
\end{eqnarray*}
\end{lemma}
The
expectation is taken over each random variable. Also, note that without the temporal dependency (i.e. $c_j^t=0$), this becomes analogous to Lemma 4.4 in \citep{bottou2018optimization}.  Naturally it follows, that
\begin{proposition}
Under Assumptions 1, 2, 3 and 4, each term of the following equation converges:

\begin{align}
\sum_{t=0}^T \eta_t
\mathbb{E}[\Vert\nabla\mathcal{L}(\Theta_j^t)\Vert^2] &\leq  \mathbb{E}[\mathcal{L}(\Theta^0_j)]\nonumber\\
&+G\sum_{t=0}^T\eta_t\left(\sqrt{2c_{j-1}^t}+\frac{L\eta_t}{2}\right)\nonumber.
\end{align}
\end{proposition}Thus the DGL scheme converges in the sense of \citep{bottou2018optimization,Huo2018}.
 We can also  obtain the following rate:
\begin{corollary} The sequence of expected gradient norm accumulates around 0 at the following rate:
\begin{equation}
\inf_{t\leq T}\mathbb{E}[\Vert \nabla\mathcal{L}(\Theta_j^t)\Vert^2]\leq\mathcal{O}\left(\frac{\sum_{t=0}^T\sqrt{c_{j-1}^t}\eta_t}{\sum_{t=0}^T\eta_t}\right)\,.\end{equation}\end{corollary}

Thus compared to the sequential case, the parallel setting adds a delay that is controlled by  $\sqrt{c_{j-1}^t}$.



\begin{figure*}[t]
    \centering
    \includegraphics[scale=0.16]{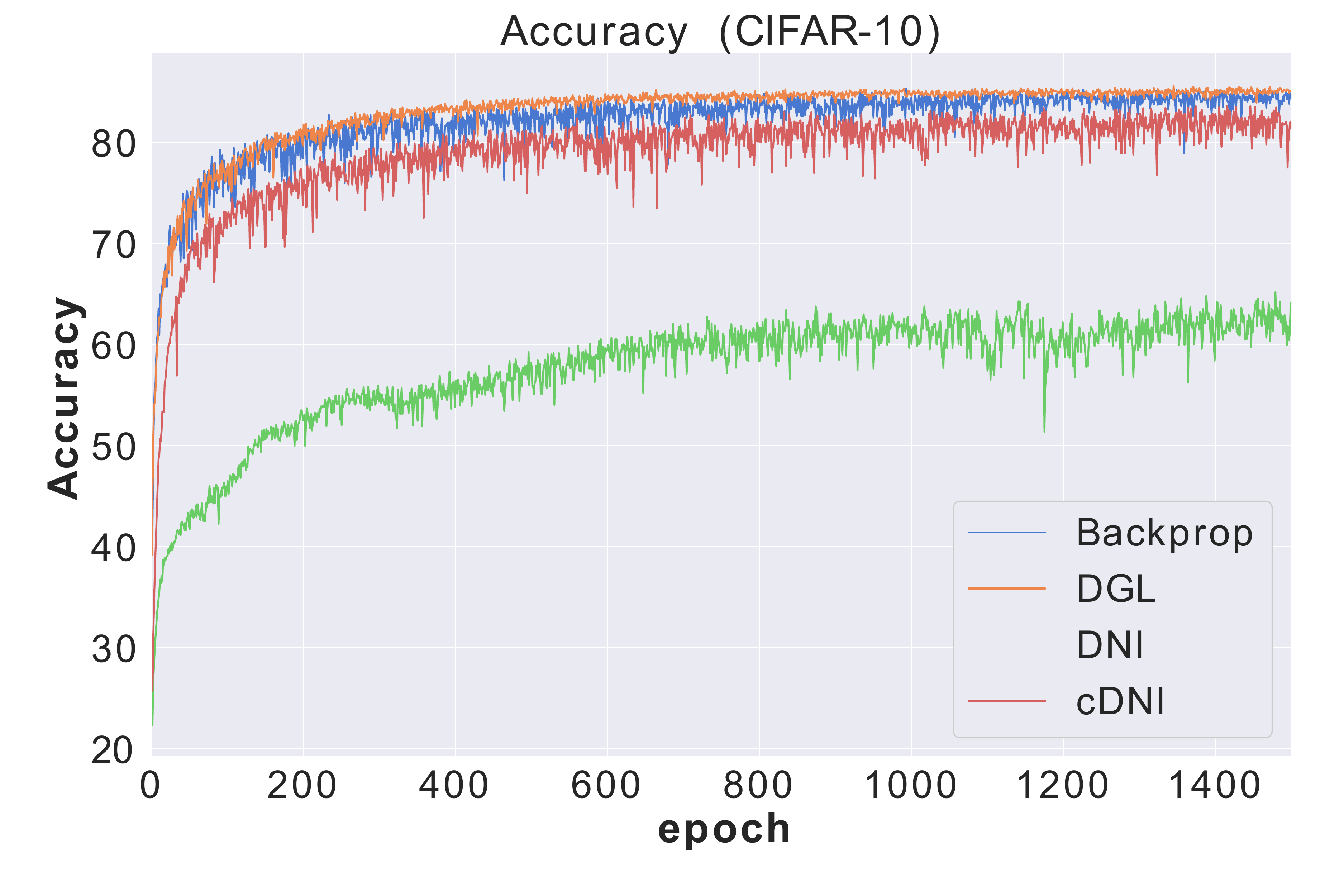}\includegraphics[scale=0.16]{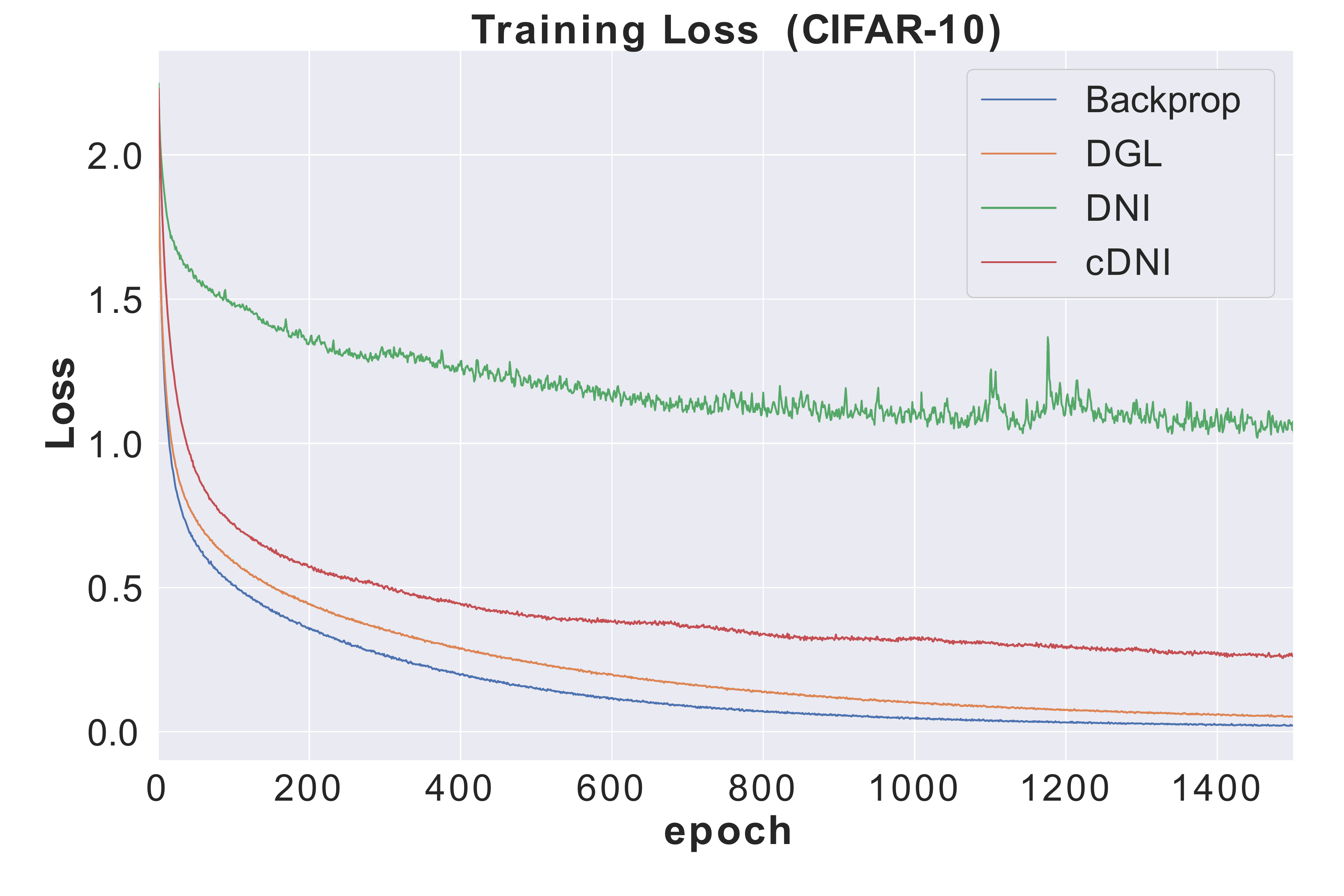}
    \caption{Comparison of DNI, cDNI, and DGL in terms of training loss and test accuracy for experiment from \citep{jaderberg2016decoupled}. DGL converges better than cDNI and DNI with the same auxiliary net. and generalizes better than backprop.}
    \label{fig:dni_comp}
\end{figure*}

\section{Experiments}
We conduct experiments that empirically show that DGL optimizes the greedy objective well, showing that it is favorable 
to
recent state-of-the-art proposals for decoupling training of deep network modules. We show that unlike previous decoupled proposals it can still work on a large-scale dataset (ImageNet) and that it can, in some cases, generalize better than standard back-propagation. We then extensively evaluate the asynchronous
version of
DGL, simulating large delays. For all experiments we use architectures taken from prior works and standard optimization settings. 

\subsection{Other Approaches and Auxiliary Network Designs}\label{sec:comparisons}
This section presents experiments evaluating DGL with the CIFAR-10 dataset~\citep{krizhevsky2009learning} and standard data augmentation. We  first use a  setup that permits us to compare against the DNI method and which also highlights the generality and scalability of DGL.  We then consider the design of a more efficient auxiliary network which will help to scale to the ImageNet dataset. We will also show that DGL is  effective at optimizing the greedy objective compared to a naive sequential algorithm.
\paragraph{Comparison to DNI} 
We reproduce the CIFAR-10 CNN experiment described in \citep{jaderberg2016decoupled}, Appendix C.1. This experiment utilizes a 3 layer network with auxiliary networks of 2 hidden CNN layers.  We compare our reproduction to the DGL approach. Instead of the final synthetic gradient prediction for the DGL we apply a final projection to the target prediction space. Here, we follow the prescribed optimization procedure from \citep{jaderberg2016decoupled}, using Adam with a learning rate of $3\times 10^{-5}$. We run training for 1500 epochs and compare standard backprop, DNI, context DNI (cDNI) \citep{jaderberg2016decoupled} and DGL.  Results are shown in Fig. \ref{fig:dni_comp}. Details are included in the Appendix. The DGL method outperforms DNI and the cDNI by a substantial amount both in test accuracy and training loss. Also in this setting, DGL can generalize better than standard backprop and obtains a close final training loss.

We also attempted DNI with the more commonly used optimization settings for CNNs (SGD with momentum and step decay), but found that DNI would diverge when larger learning rates were used, although DGL sub-problem optimization worked effectively with common CNN optimization strategies. We also note that the prescribed experiment uses a setting where the scalability of our method is not fully exploited. Each layer of the primary network of \citep{jaderberg2016decoupled} has a pooling operation, which permits the auxiliary network to be small for synthetic gradient prediction. This however severely restricts the architecture choices in the primary network to using a pooling operation at each layer. In DGL, we can apply the pooling operations in the auxiliary network, thus permitting the auxiliary network to be negligible in cost even for layers without pooling (whereas synthetic gradient predictions  often have to be as costly as the base network). Overall, DGL is more scalable, accurate and  robust to changes in optimization hyper-parameters than DNI.

\paragraph{Auxiliary Network Design} We consider different auxiliary networks for CNNs. As a baseline we use convolutional auxiliary layers as in \citep{jaderberg2016decoupled} and \citep{shallow}. For distributed training application this approach is sub-optimal as the auxiliary network can be substantial
in size
compared to the base network, leading to poorer parallelization gains. We note however that even in those cases (that we don't study here) where the auxiliary network computation is potentially on the order of the primary network, it can still give advantages for parallelization for very deep networks and many available workers. 

The primary network architecture we use for this study is a simple CNN similar to VGG family
of
models \citep{simonyan2014very} and those used in \citep{shallow}. It consists of 6 convolutions of size $3\times 3$, batchnorm and shape preserving padding, with $2\times2$ maxpooling at layers 1 and 3. The  width of the first layer is 128 and is doubled at each downsampling operation. The final layer does not have an auxiliary model-- it is followed by a pooling and 2-hidden layer fully connected network, for all experiments.  Two alternatives to the CNN auxiliary of \citep{shallow} are explored (Tab. \ref{tab:flop}).
    \begin{table}
    \centering
    \begin{tabular}{|c|c|c|}
    \hline&Relative FLOPS&Acc.\\\hline
     CNN-aux&   $200\%$ &92.2\\\hline
    MLP-aux&     $0.7\%$ &90.6\\\hline
     MLP-SR-aux& $4.0\%$  &91.2\\\hline

    \end{tabular}
    \caption{Comparison of auxiliary networks on CIFAR. CNN-aux applied in previous work is inefficient w.r.t. the primary module. We report flop count of the aux net relative to the largest module. MLP-aux and MLP-SR-aux applied after spatial averaging operations are far more effective with min. acc. loss. }
    \label{tab:flop}
    \end{table}

The baseline auxiliary strategy based on  \citep{shallow} and \citep{jaderberg2016decoupled} applies 2 CNN layers followed by a $2\times 2$ averaging and projection, denoted as \textit{CNN-aux}. 
First, we explore  a direct application of the spatial averaging to $2\times2$ output shape (regardless of the  resolution) followed by a 3-layer MLP (of constant width). This is denoted \textit{MLP-aux} and drastically reduces the FLOP count with minimal accuracy loss compared to \textit{CNN-aux}. Finally, we study a staged spatial resolution, first reducing the spatial resolution by 4$\times$ (and total size 16$\times$), then applying 3 $1\times 1$ convolutions followed by a reduction to $2\times 2$ and a 3 layer MLP, that we denote as \textit{MLP-SR-aux}. These latter two strategies that leverage the spatial averaging produce auxiliary networks that are less than $5\%$ of the FLOP count of the primary network even for large spatial resolutions as in real world image datasets. We will show that MLP-SR-aux is still effective even for the large-scale ImageNet dataset. We note that these more effective auxiliary models are not easily applicable in the case of DNI's gradient prediction. \paragraph{ Sequential vs. Parallel Optimization of Greedy Objective}
We briefly compare the sequential optimization of the greedy objective \citep{shallow,bengio2007greedy} to the DGL (Alg.  \ref{algo:basic}). We use a 6 layer CIFAR-10 network with an MLP-SR-aux auxiliary model. In parallel we train the layers together for 50 epochs and in the sequential training we train each layer for 50 epochs before moving to the subsequent one. Thus the difference to DGL lies only in the input received at each layer (fully converged previous layer versus not fully converged previous layer). The rest of the optimization settings are identical. Fig.~\ref{fig:greedy_v_parallel} shows comparisons of the learning curves for sequential training and DGL at layer 4 (layer 1 is the same for both as the input representation is not varying over the training period).  DGL quickly catches up with the sequential training scheme and appears to sometimes generalize better. Like \citet{oyallon2017building}, we also visualize the dynamics of training per layer in Fig.  \ref{fig:dynamics}, which demonstrates that after just a few epochs the individual layers build a dynamic of progressive improvement with depth. 
\begin{figure}[t]
\center
\includegraphics[scale=0.17]{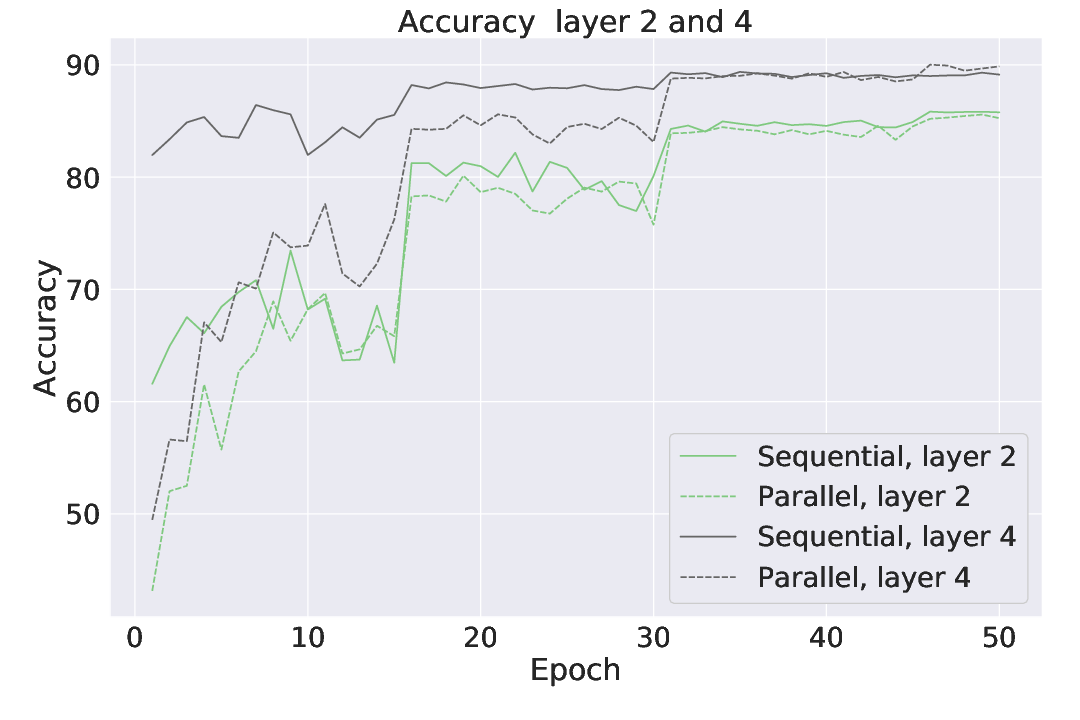}
    \caption{Comparison of sequential and parallel training. Parallel  catches up rapidly to sequential.}\label{fig:greedy_v_parallel}
\end{figure}
\begin{figure}
\center
    \includegraphics[scale=0.17]{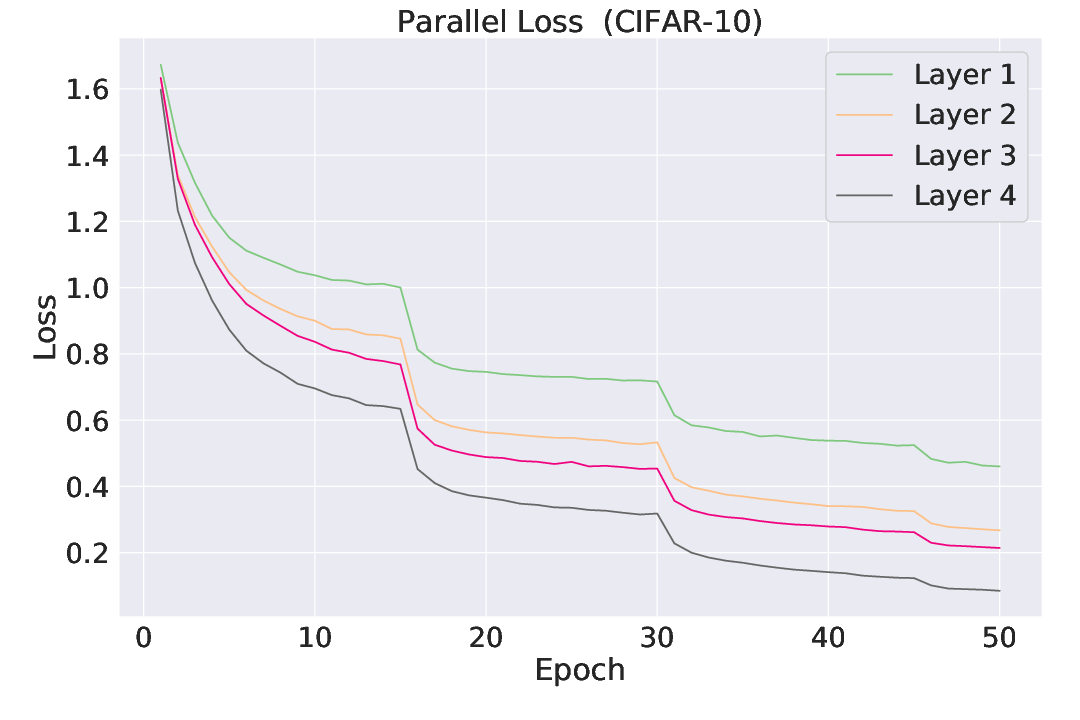}
    
    \caption{Per layer-loss on CIFAR: after few epochs,  the layers build a dynamic of progressive improvement in depth.}
    \label{fig:dynamics}

\end{figure}

\paragraph{Multi-Layer modules}
We have so far mainly considered the setting of layer-wise decoupling. This approach however can easily be applied to generic modules. Indeed, approaches such as  DNI \citep{jaderberg2016decoupled} often consider decoupling entire multi-layer modules. Furthermore the propositions for backward unlocking \citep{ddg,Huo2018} also rely on and report they can often only decouple 100 layer networks into 2 or 4 blocks before observing optimization issues or performance losses and require that the number of parallel modules be much lower than the network depth for the theoretical guarantees to hold. As in those cases, using  multi-layer decoupled modules can improve performance and is natural in the case of deeper networks.  We now use such a multi-layer approach to directly compare to the backward unlocking of \citep{ddg} and then subsequently we will apply this on deep networks for ImageNet. From here on we will denote $K$ the number of total modules a network is split into.


\paragraph{Comparison to DDG} \citet{ddg} propose a solution to the backward locking (less efficient than solving update-locking, see discussion in Sec~\ref{sec:rel}). 
\begin{table}
    \center\begin{tabular}{|c|c|c|}
        \hline
           Backprop &  DDG & DGL  \\\hline
           93.53 &  93.41 & $93.5\pm0.1$ \\\hline
    \end{tabular}
    \caption{ResNet-110($K=2$) for Backprop and DDG method  from \citep{ddg}. DGL is run for 3 trials to compute variance. They give the same acc. with DGL being update unlocked, DDG only backward unlocked. DNI is reported to not work in this setting \citep{ddg}.}
    \label{tab:ddg_comp}
\end{table}
We show that even in this situation the DGL method can provide a strong  baseline for work on backward unlocking. We take the experimental setup from \citep{ddg}, which considers a ResNet-110 parallelized into $K=2$ blocks. 
We use the auxiliary network MLP-SR-aux which has less than $0.1\%$ the FLOP count of the primary network. We use the exact optimization and network split points as in \citep{ddg}. 
 
To assess variance in  CIFAR-10 accuracy, we perform 3 trials. Tab. \ref{tab:ddg_comp} shows that the accuracy is the same across the DDG method, backprop, and our approach. DGL achieves better parallelization because it is update unlocked. We use the parallel implementation provided by \citep{ddg} to obtain a direct wall clock time comparison. We note that there are multiple considerations for comparing speed across these methods (see Appendix C). 

\paragraph{Wall Time Comparison} We compare to the parallel implementation of \citep{ddg} using the same communication protocols and run on the same hardware. We find for $K=2,4$ GPU gives a $~5\%, 18\%$ respectively speedup over DDG. With DDG $K=4$ giving approximately $2.3\times$ speedup over standard backprop on same hardware (close to results from \citep{ddg}).

\subsection{Large-scale Experiments}\label{sec:imagenet}

\begin{figure*}[t]
\begin{minipage}{\textwidth}
\begin{minipage}[b]{0.49\textwidth}
\centering
\includegraphics[scale=0.18]{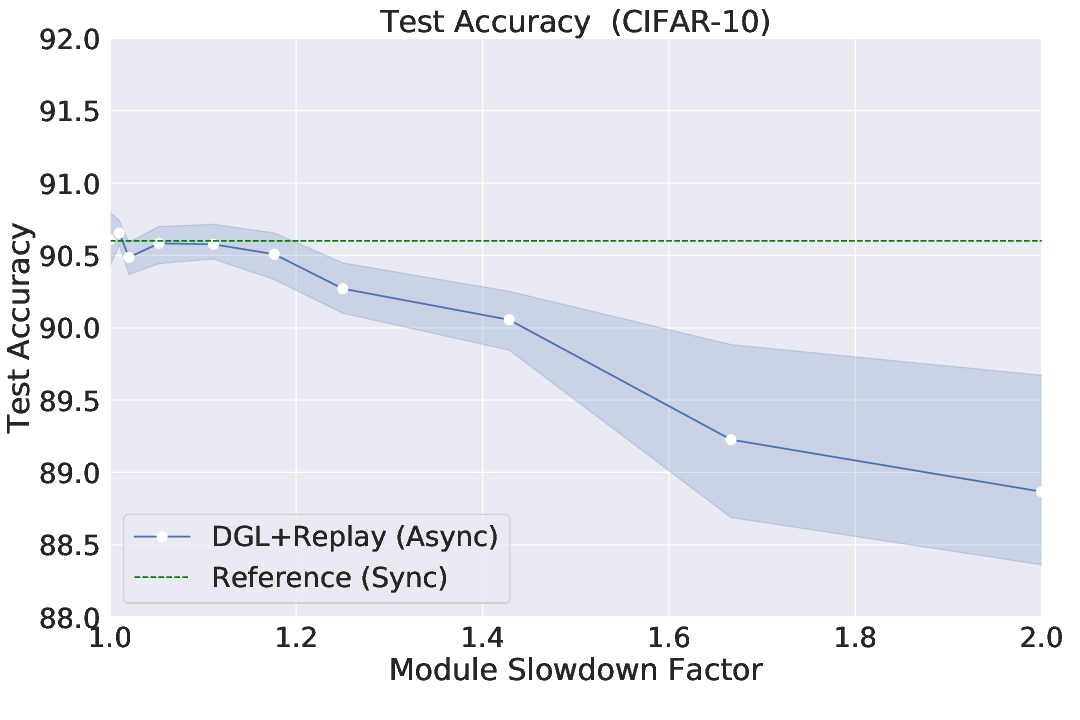}
\caption{Evaluation of Async DGL. A single layer is slowed down on average over others, with negligible losses of accuracy at even substantial delays.\label{fig:buffer}}
\end{minipage}
\hfill
\begin{minipage}[b]{0.48\textwidth}
\resizebox{.95\textwidth}{!}{\begin{tabular}{c|c|c}
Model (training method) & Top-1 & Top-5 \\\hline
VGG-13 (DGL per Layer, $K=10$) & 64.4 & 85.8 \\\hline
VGG-13 (DGL $K=4$) & \textbf{67.8}  & \textbf{88.0}\\\hline
VGG-13 (backprop) & 66.6 & 87.5 \\\hline\hline
VGG-19 (DGL $K=4$) & 69.2  & 89.0 \\\hline
VGG-19 (DGL $K=2$) & \textbf{70.8}  & \textbf{90.2} \\\hline
VGG-19 (backprop) & 69.7 & 89.7 \\\hline\hline
ResNet-152 (DGL $K=2$) & \textbf{74.5} &  92.0 \\\hline
ResNet-152 (backprop) & 74.4 &\textbf{92.1}\\\hline
\end{tabular}
}
\captionsetup{type=table} 
  \caption{ImageNet results using training schedule of \citep{Xiao2019} for DGL and standard e2e backprop. DGL with VGG and ResNet obtains similar or better accuracies, while enabling parallelization and reduced memory. \label{tab:imagenet_results}}
\end{minipage}
\end{minipage}
\end{figure*}
Existing methods considering update or backward locking have not been evaluated on large image datasets as they are often unstable or already show large losses in accuracy on smaller datasets. Here we study the optimization of several well-known architectures, mainly the VGG family \citep{simonyan2014very} and the ResNet \citep{he2016deep}, with DGL on the ImageNet dataset. 
In all our experiments we use the MLP-SR-aux auxiliary net which scales well from the smaller CIFAR-10  to the larger ImageNet. The final module has no auxiliary network. For all optimization of auxiliary problems and for end-to-end optimization of reference models we use the shortened optimization schedule prescribed in \citep{Xiao2019}.  Results are shown in Tab.~\ref{tab:imagenet_results}. We see that for all the models DGL can perform as well and sometimes better than the end-to-end trained models, while permitting parallel training. In all these cases the auxiliary networks are neglibile (see Appendix Table 4 for more details). For the VGG-13 architecture  we also evaluate the case where the model is trained layer by layer ($K=10$). Although here performance is slightly degraded, we find it is suprisingly high given that no backward communication is performed. We conjecture that improved auxiliary models and combinations with methods such as \citep{Huo2018} to allow feedback on top of the local model, may further improve performance. 
Also for the settings with larger potential parallelization, slower but more performant auxiliary models could potentially be considered as well.

The synchronous DGL has also favorable memory usage compared to DDG and to the DNI method, DNI requiring to store larger activations and DDG having high memory compared to the base network even for few splits \citep{Huo2018}. Although not our focus, the single worker version of DGL has favorable memory usage compared to standard backprop training. For example, the ResNet-152 DGL $K=2$ setting  can fit $38\%$ more samples on a single 16GB GPU than the  standard end-to-end training.

\subsection{Asynchronous DGL with Replay}\label{replay}
\begin{figure}
\begin{center}
\includegraphics[scale=0.4]{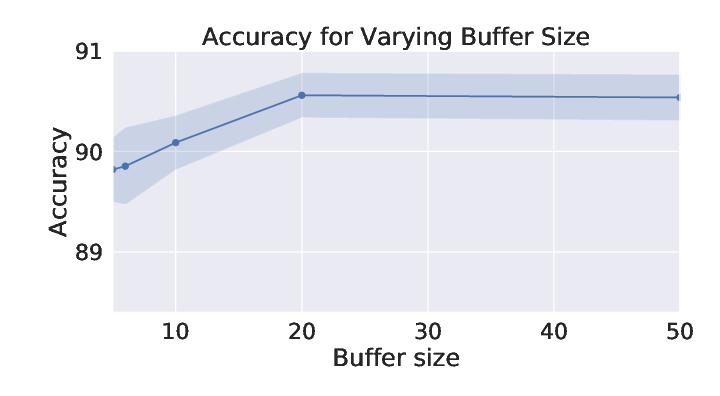}
\end{center}
\caption{Buffer size vs. Acc. for Async DGL. Smaller buffers produce only small loss in acc.}\label{fig:buffer2}\end{figure} 

We now study the effectiveness of
Alg. \ref{algo:buffer_sym} with respect to delays. We use a 5 layer CIFAR-10 network with the MLP-aux and with all other architecture and optimization settings as in the auxiliary network experiments of Sec.~\ref{sec:comparisons}.  
Each layer is equipped with a buffer of size $M$ samples.
At each iteration, a layer is chosen according to the pmf $p(j)$, and a batch selected from buffer $B_{j-1}$.
One layer is slowed down by decreasing its selection probability in the pmf $p(j)$ by a factor $S$.
We define $S = \frac{p}{p(j)}$, 
where $p$ is the constant probability of any other worker being selected, so $p = (1-p(j)) / (J - 1)$.
Taken together his implies that:
\begin{equation}
    S=\frac{1}{J-1}\left(\frac{1}{p(j)}-1\right)\label{slowdown}
\end{equation}
We evaluate
different slowdown factors (up to $S=2.0$). Accuracy versus $S$ is
shown in Fig. \ref{fig:buffer}. For this experiment we use a buffer of size $M=50$ samples. We run separate experiments with the slowdown applied at each 
of the 6 layers
of the network as well as 3 
random
seeds for each of these settings (thus 18 experiments per data point). We show the evaluations for 10 values of $S$.
To ensure 
a fair comparison we stop updating layers once they have completed 50 epochs, ensuring an identical number of gradient updates for all layers in all experiments.

In practice one could continue updating until all layers are trained. In Fig.~\ref{fig:buffer} we compare to the synchronous case. First, observe that the accuracy of the synchronous algorithm is maintained in the setting where $S=1.0$ and the pmf is uniform. Note that even this is a non-trivial case, as it will mean that layers inherently have random delays (as compared to  Alg. ~\ref{algo:basic}). Secondly, observe that accuracy is maintained until approximately $1.2\times$ and accuracy losses after that the difference remains small.
Note that even case $S=2.0$ is somewhat drastic: for 50 epochs, the slowed-down layer is only on epoch 25 while those following it are at epoch 50.

We now consider the
performance with respect to the buffer size. Results are shown in Fig.~\ref{fig:buffer2}. For this experiment we set $S=1.2\times$. Observe that even a tiny buffer size can yield only a slight loss in 
accuracy. Building on this demonstration there are multiple directions to improve Async DGL with replay.
For example improving the efficiency of the buffer~\cite{oyallon2018compressing}, by including data augmentation in feature space \citep{verma2018manifold}, mixing samples in batches, or improved batch sampling, among others.

\subsection{
Adaptive
Online 
Compression for Async-DGL with Replay}
\label{sec:xp_qt}



{We now evaluate the online compression proposed in Sec.~\ref{subsec:qt}. We consider the setting of the previous Sec. \ref{replay}. In the subsequent experiments on CIFAR-10 we will consider the same simplified CNN architecture, however to simplify the analysis we will focus on just a 4 module version of this model which will allow us to explore in more depth the behavior of the proposed approach.

Similarly to the previous section on Aync-DGL, we slow down the communication between selected layers. This permits to consider a scenario where each layer is hosted on a different node. In such a setting the communication bandwidth would naturally be bounded. We follow the same training strategy as above: Each layer is asynchronously optimized using SGD with an initial learning rate of 0.1 (with a decay factor
of $0.2$ every 15 epochs),  momentum of 0.9, a weight decay equal to $5\times 10^{-4}$, and a cross-entropy loss.After completing the equivalent of 50 epochs each layer stops performing updates and only passes signals to upper layers. As in the previous section we use a LIFO priority 
rule with a penalty for reuse. In all our experiments, results are reported as an averaging over 5 seeds. 

}

\begin{figure}[ht]
  \subfigure[Communication lag between layer 1-2]{
	\begin{minipage}[c][1\width]{
	   0.31\textwidth}
	   \centering
	   \includegraphics[width=1\textwidth]{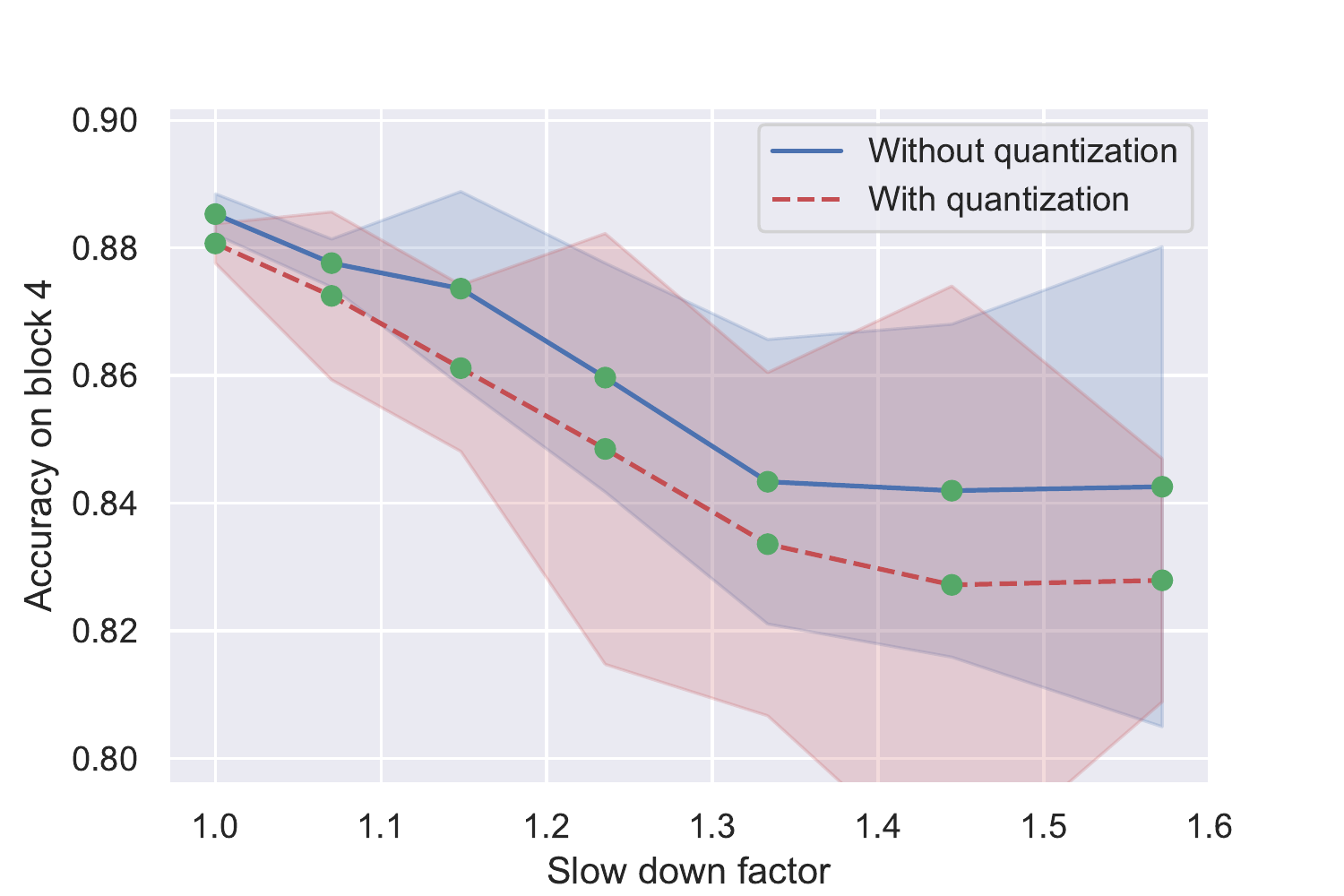}
	\end{minipage}}
 \hfill 	
  \subfigure[Communication lag between layer 2-3]{
	\begin{minipage}[c][1\width]{
	   0.31\textwidth}
	   \centering
	   \includegraphics[width=1\textwidth]{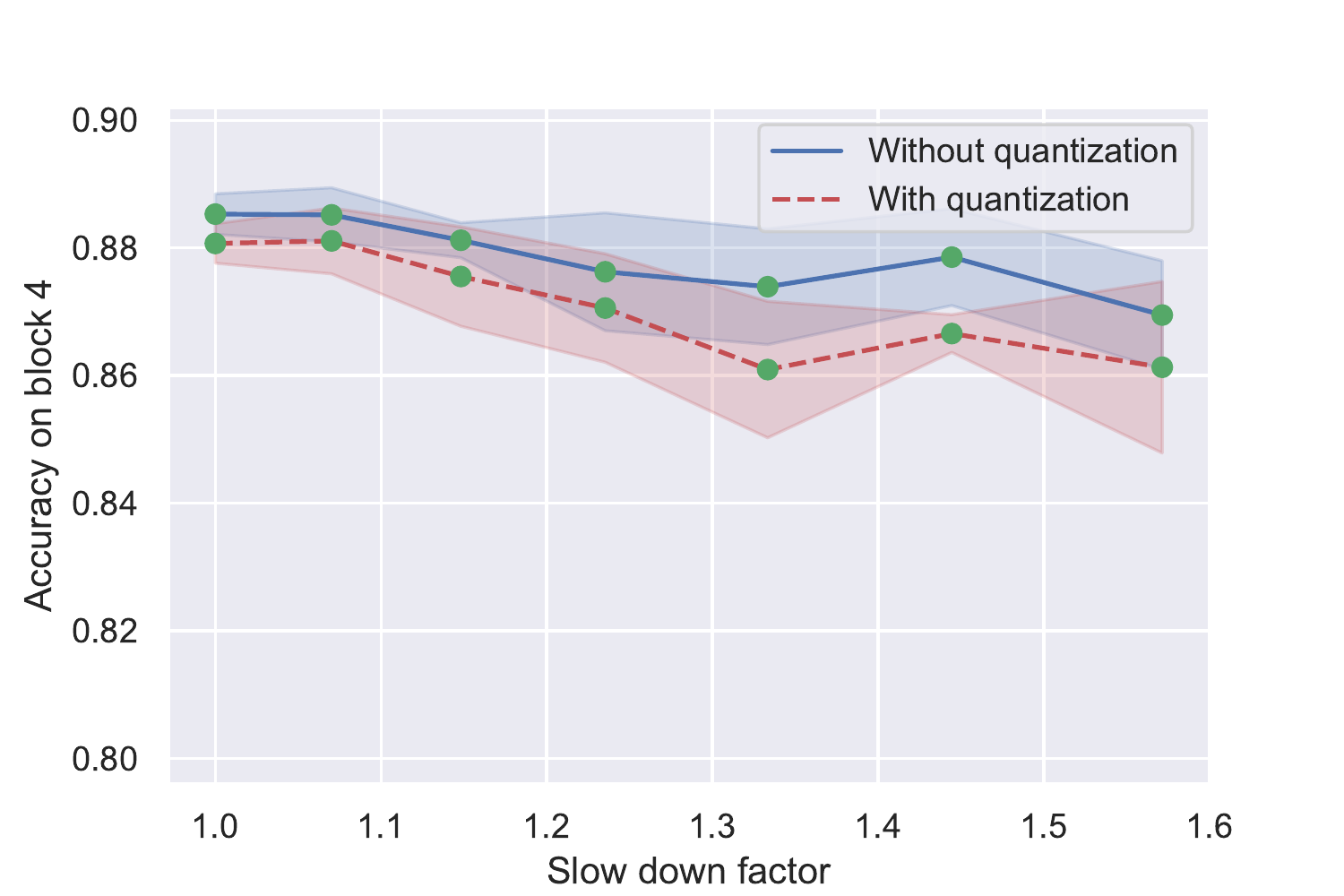}
	\end{minipage}}
 \hfill	
  \subfigure[Communication lag between layer 3-4]{
	\begin{minipage}[c][1\width]{
	   0.31\textwidth}
	   \centering
	   \includegraphics[width=1\textwidth]{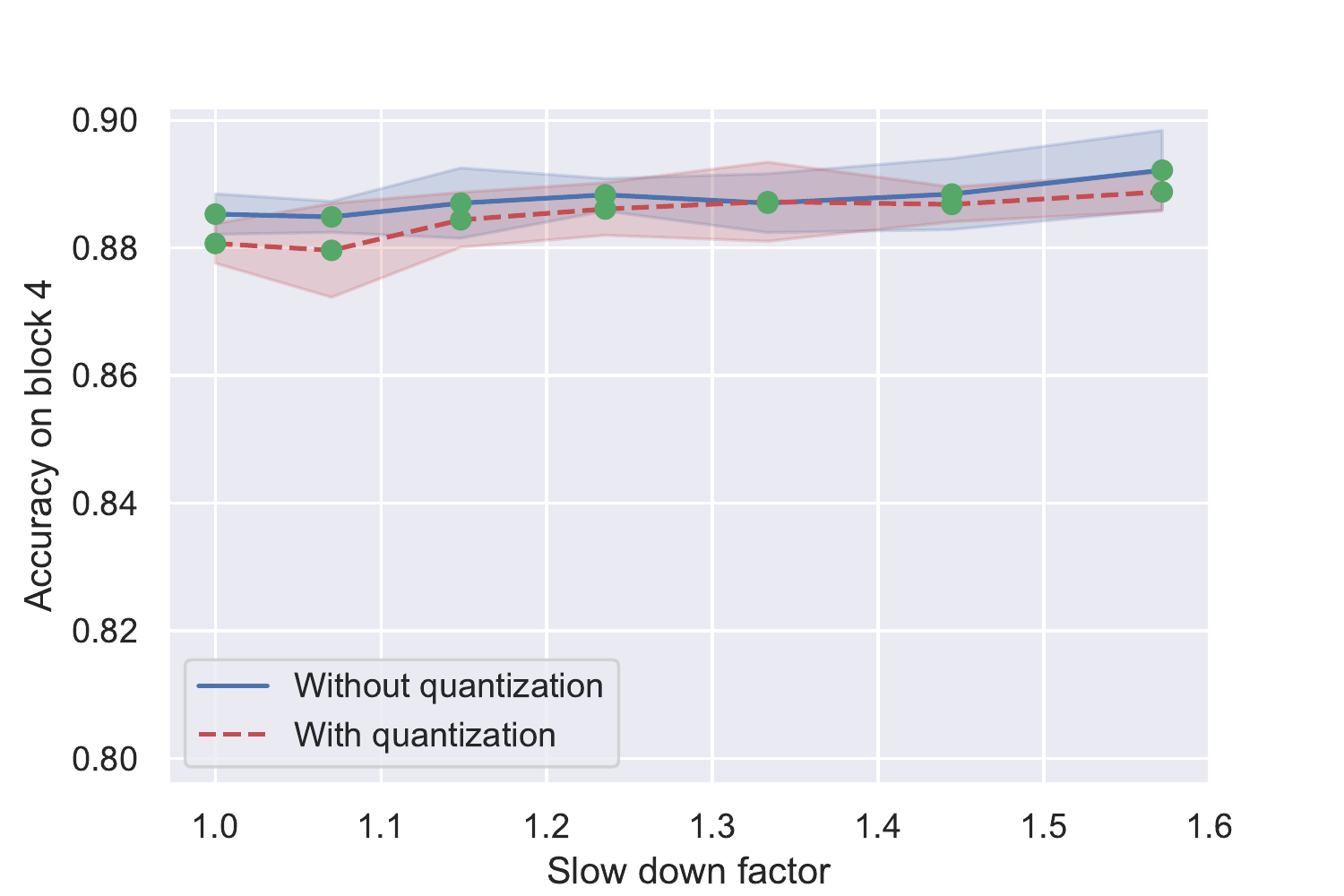}
	\end{minipage}}
\caption{{Comparison between quantization and no quantization: Accuracy when the communication between nodes $i,i+1$ is slowed down  for the layers $i=1,2,3$. 
Using quantization
we observe
a drop in performance when
the slow-down is
applied on early layers, but we also note that the slow-down has positive impact when applied to the layer 3. This is 
expected:
layer slow-downs will 
affect
the distribution of the subsequent layers, except if applied to the final layer which can then learn on top of almost converged features.}
}
\label{fig:qt_noise}
\end{figure}

We follow the adaptive quantization procedure proposed in Sec.  \ref{subsec:qt}  that encodes and decodes a modules local memory with a learned set of codebooks. With respect to this buffered memory and quantization procedure all experiments are conducted with a buffer of size $M=2$ batches (note the batch size is 128). The number of codebooks with 256 atoms, and a batch size of 128, unless
otherwise stated. In all our experiments, we kept constant the number of different codebooks at 32. Each setting was independently run with 5 random seeds: we report the average accuracy with its corresponding confidence interval at $95\%$ level.

The results in this setting are illustrated in Fig. \ref{fig:qt_noise}. Here we apply a slow-down factor in the range {$[1, 1.6]$} to each layer, as defined in \eqref{slowdown} and we keep all the other hyper-parameters fixed. We compare the numerical accuracy with and without quantization. Compared to the previous section, we refine our study by reporting the accuracy at  every depth instead of averaging the accuracies for each given slow-down factor regardless of the position at which it was applied.

Note that, thanks to the quantization module, the communication bandwidth is reduced by a factor $15.5, 23.3, 21.3$ respectively, and the buffer memory is reduced with a factor in $15.8, 28.4, 28.4$ respectively. 
We observe there are small, but potentially acceptable, accuracy losses with quantization depending on where the delay is induced. If it is in the first 2 layers the difference is typically less than $1.5\%$. This improves in cases where the delay is induced at a layer. We hypothesize the 3rd module shows less performance gains when slowed down as it receives features from the previous layers which are potentially changing less dramatically and are more stable. We note that even a $1.5\%$ accuracy loss can be compensated by ability to train much larger models due to parallelization thereby leading to potentially higher accuracy models overall. The results suggest that our Quantized Async-DGL is robust to the distribution shifts incurred by the lag. We now ablate various aspects of the Quantized Async DGL. 
\begin{figure}
\begin{center}
\includegraphics[scale=0.6]{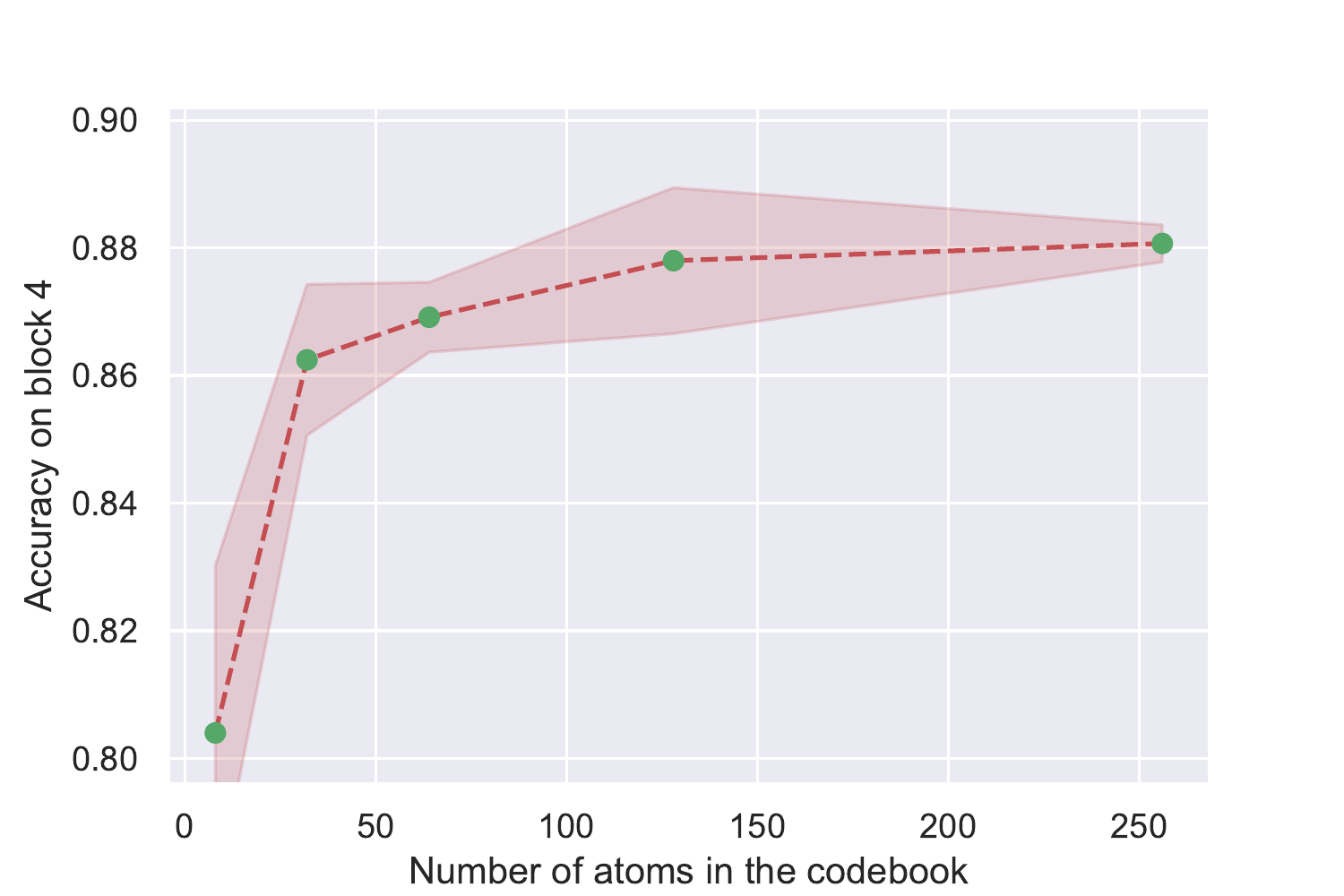}
\end{center}
\caption{Accuracy at the final layer of our model as a function of the number of atoms. Only 7 bits are necessary to obtain our top accuracies.
}
\label{fig:codebook}
\end{figure}

\paragraph{Codebook size} We 
study the impact of the codebook size on the accuracy of a quantized buffer.  
As expected, the number of atoms in the compression codebook is correlated with performance for small codebook sizes and it 
reaches a plateau
when the number of atoms is large enough. Indeed, the more items in our dictionary, the better is the approximation of the quantization. 
This means that for a large number of atoms, the accuracy behaves similarly to the non-compressed version,
whereas with a small number of atoms, the accuracy depends on the compression rate.
For these experiments, no nodes are artificially slowed down.
Our results in Fig. \ref{fig:codebook} illustrate the robustness of our approach. Indeed 128 or 256 atoms are enough to reach standard performances 
(256 atoms can be encoded using 8 bits,
which is
to be compared with the 32 bits 
single floating point
precision). The bandwidth compression 
rate
depends on the layer, but is always above 12 in this experiment. Despite this compression,
we can still recover accuracies at the level of those obtained in Sec. \ref{algo:buff_para}.

\begin{figure}
\begin{center}
\includegraphics[scale=0.7]{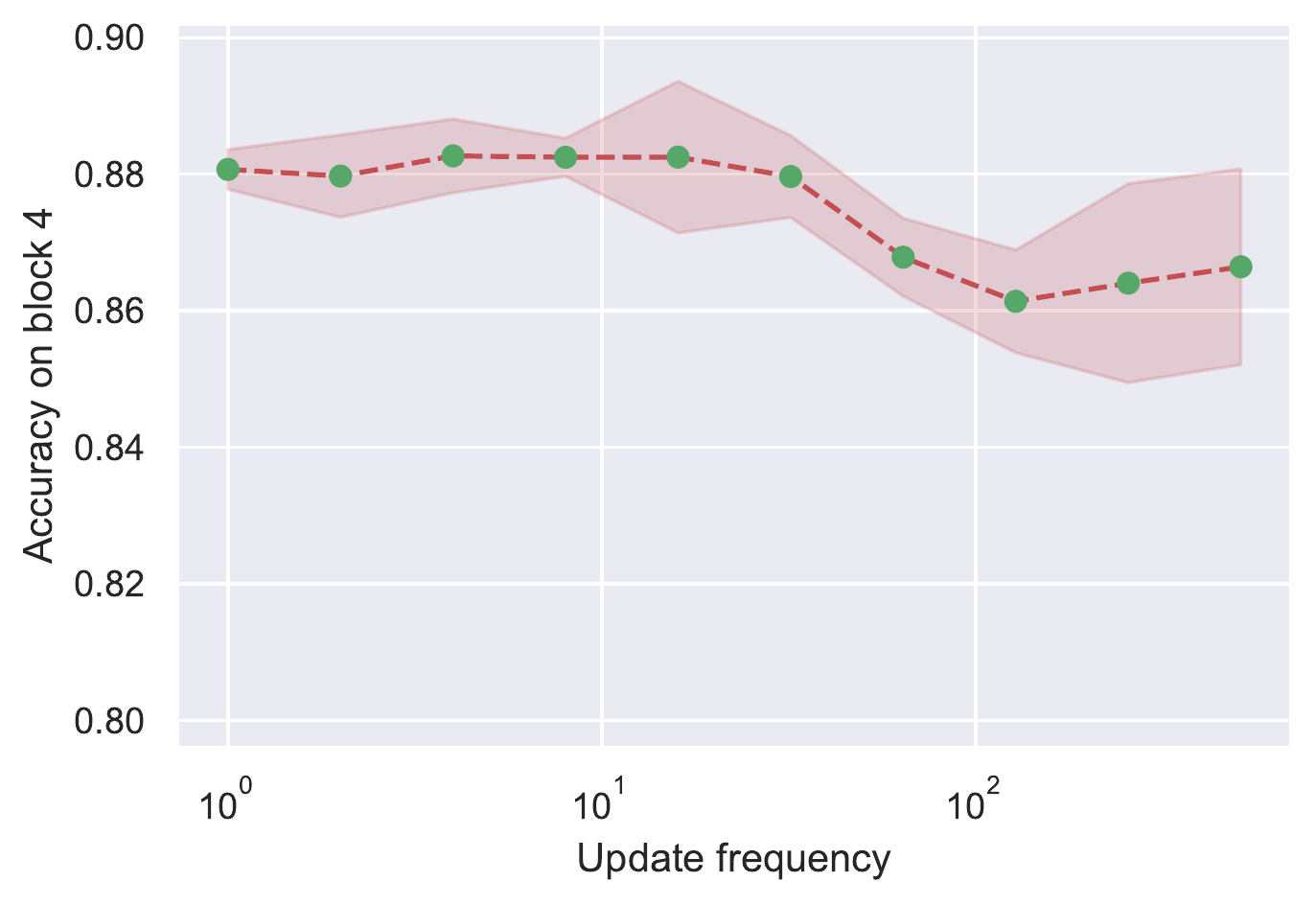}
\end{center}
\caption{Accuracy of the final block as a function of codebook update frequency. Updating the atoms only one out of 16 iterations still preserves the final accuracy with almost no degradation, while leading to substantial bandwidth saving. In fact, updating the codebook only once every 400 iterations leads to a drop of only $2\%$. This is surprising and indicates that the codebook can be largely ignored when measuring the bandwidth needed for this method
}
\label{fig:upd_freq}
\end{figure} 

{\paragraph{Update Frequency} The codebook size
has to be 
added to
the bandwidth use, because a step of synchronization is required as the distribution of the activations is changing with time. Our hypothesis is that the distribution of the features will 
evolve slowly,
making it unnecessary to sync codebooks at every iteration.
We choose an update frequency $\alpha\leq 1$,
(c.f. Sec. \ref{subsec:qt})} 
and update the codebook only at a fraction $\alpha$ of the iterations.
This
makes the communications of the dictionary 
negligible
and the procedure
may only have minor impact on
the final accuracy of our model. We thus propose to study this effect: we only update the weights of the codebook every $\frac{1}{\alpha}>1$ iterations.   Fig. \ref{fig:upd_freq} shows that the accuracy of our model is extremely stable
even at low update frequency.
We note that updating the codebook every 10 iterations preserves the accuracy of the final layer. Then the accuracy drops only by $2\%$ even if the dictionaries of the quantization modules are updated only once per epoch. {The final point on our 
logarithmic
scale corresponds to a setting where the codebook is not updated
at all: 
in this particular case, the codebook is randomly initialized, and the network 
adapts
to it. This shows the strength of our method.}

\begin{figure}[ht]
  \subfigure[Communication lag between layer 1-2]{
	\begin{minipage}[c][1\width]{
	   0.3\textwidth}
	   \centering
	   \includegraphics[width=1\textwidth]{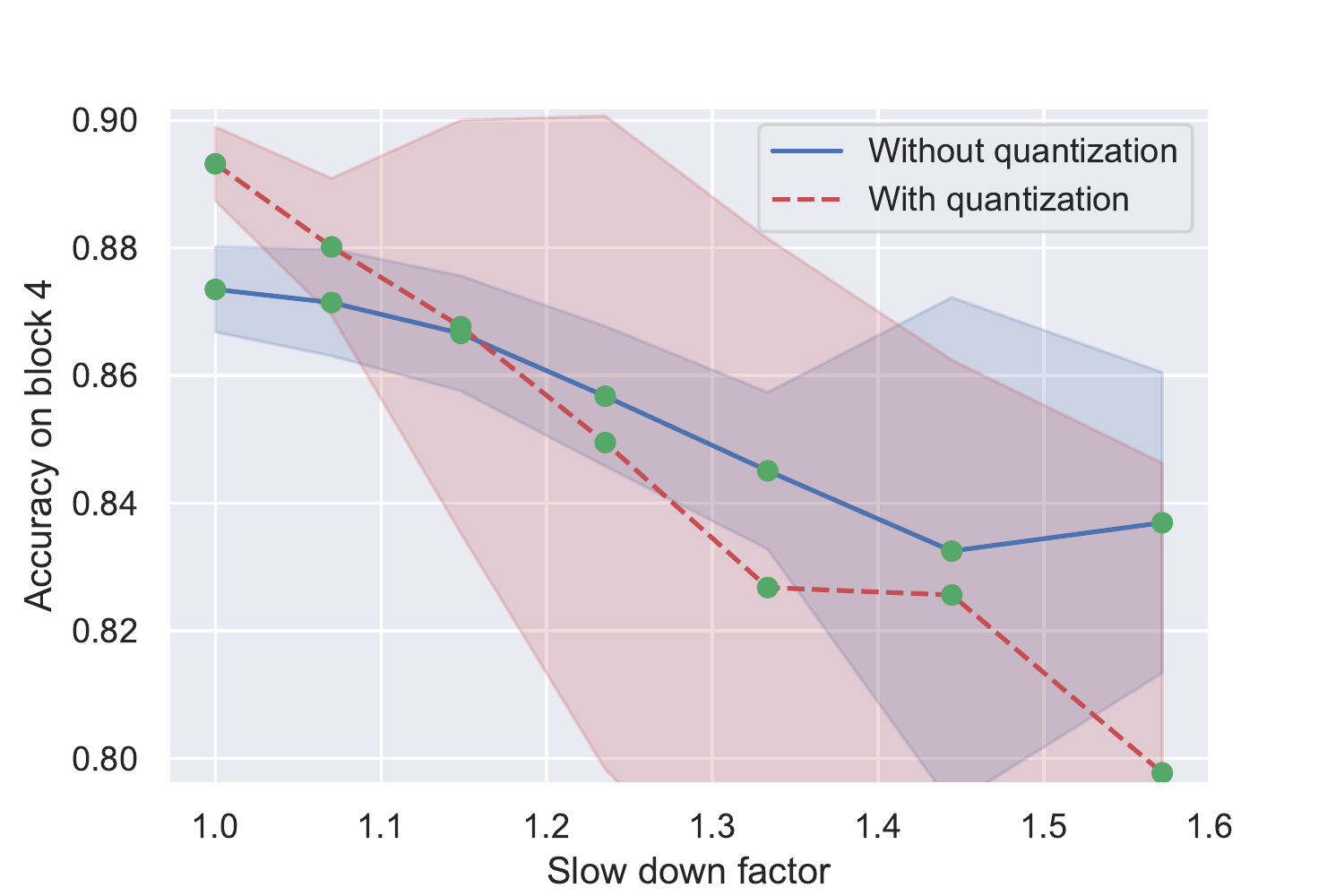}
	\end{minipage}}
 \hfill 	
  \subfigure[Communication lag between layer 2-3]{
	\begin{minipage}[c][1\width]{
	   0.3\textwidth}
	   \centering
	   \includegraphics[width=1\textwidth]{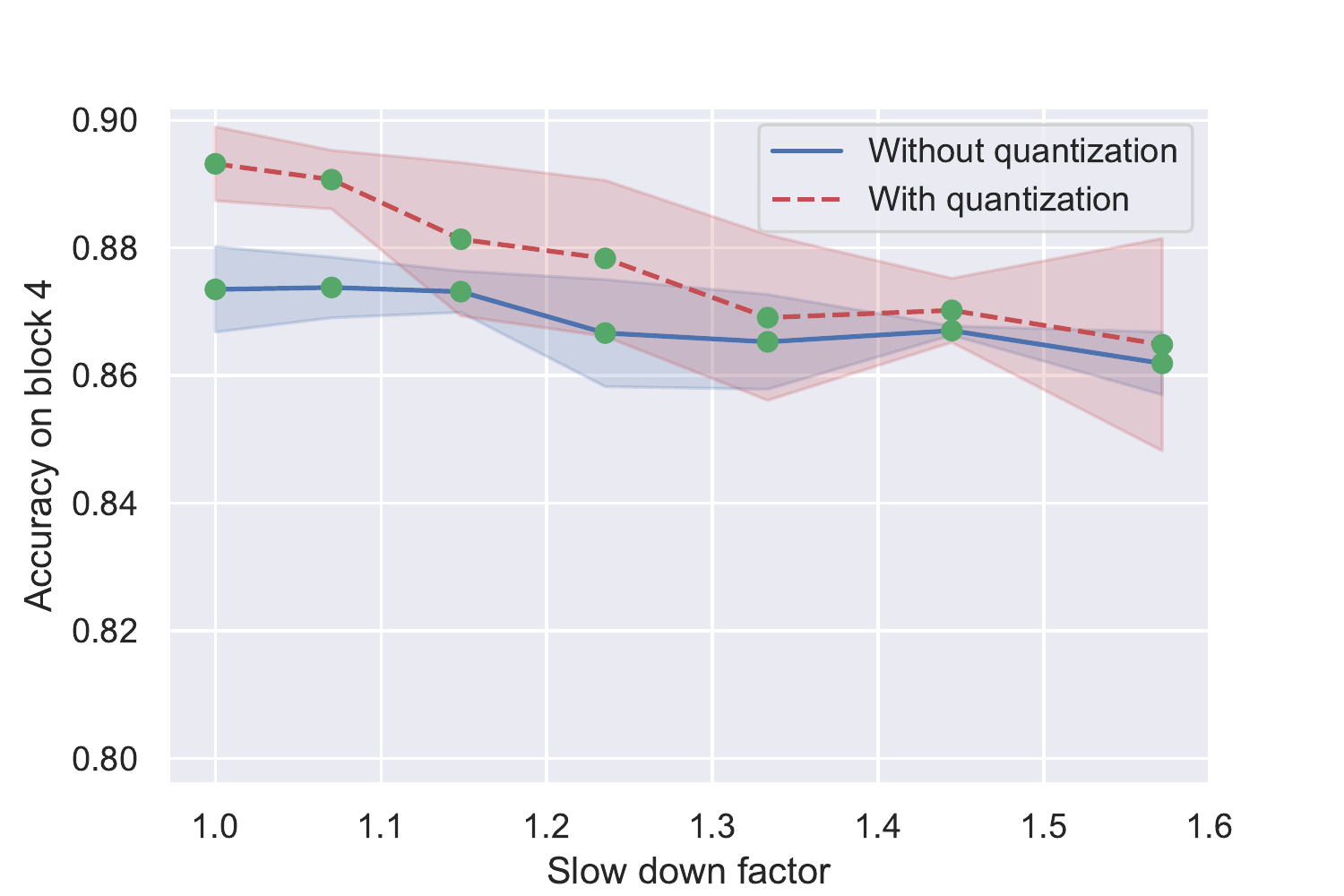}
	\end{minipage}}
 \hfill	
  \subfigure[Communication lag between layer 3-4]{
	\begin{minipage}[c][1\width]{
	   0.3\textwidth}
	   \centering
	   \includegraphics[width=1\textwidth]{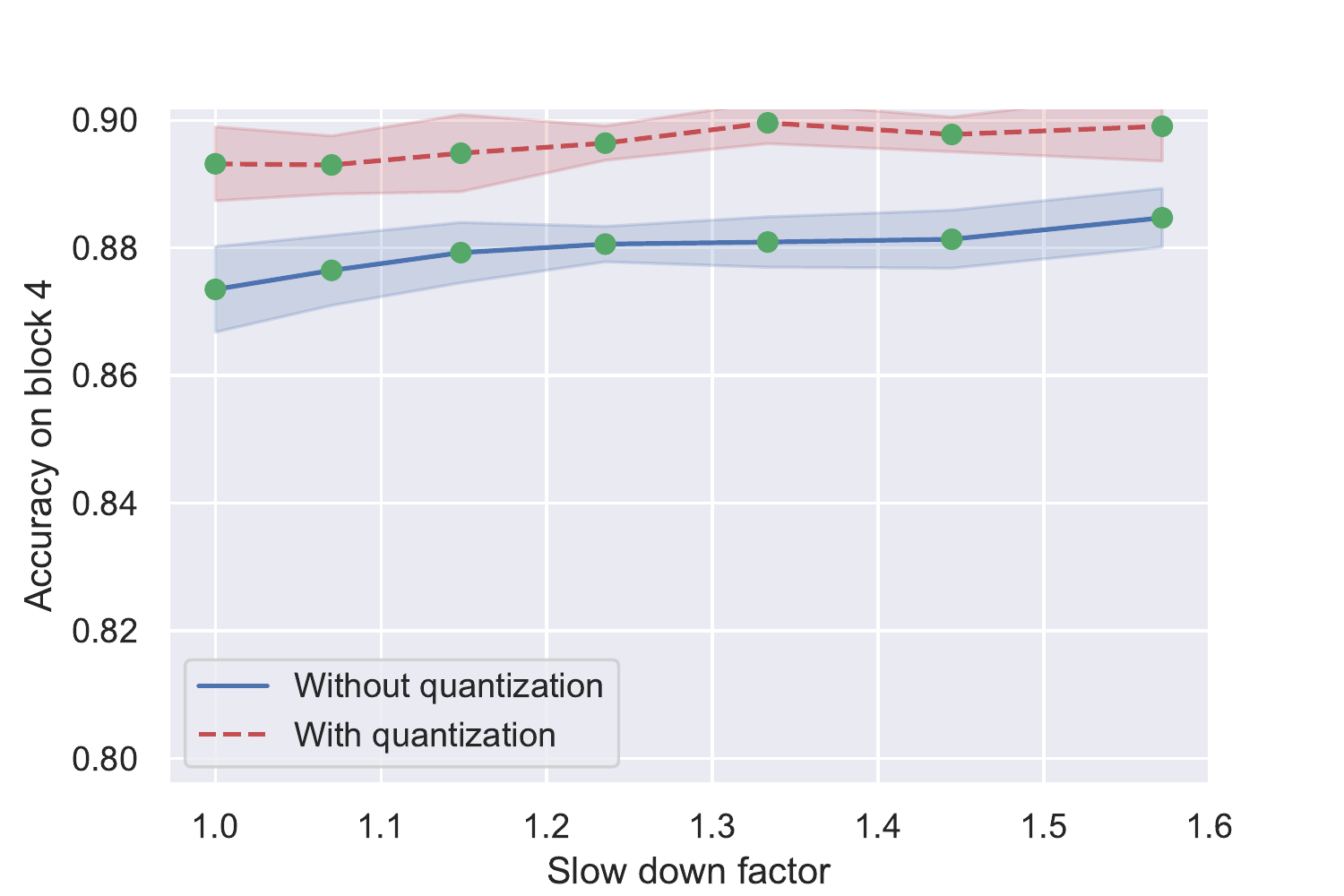}
	\end{minipage}}
\caption{We compare the test accuracy obtained for our Asynchronous method when we slow down a layer at various depths, with and without a quantization module at a \textit{fixed memory budget}. We used the hyper-parameters obtained from the Sec. \ref{sec:xp_qt} and we adjusted the size of the codebook such that the buffer memory used both with and without a quantization step is the same. Interestingly, in every settings, the algorithm which uses quantization outperforms the non-quantized version. Note, if the layer is  very out of sync, there is a substantial degradation of the accuracy. Note that slowing down the final layer as in (c) slightly boost the accuracy because this layer is fed with converged features during its last steps of training.
}
\label{fig:qt_noise2}
\end{figure}
{\paragraph{
Tradeoff between compression and buffer size at fixed memory budget
} In this 
experiment, we design a setting where each node has the same amount of limited total memory to store both the input and output 
dictionaries and the the buffer. We then compare at fixed budget the difference between quantized and standard approaches in terms of accuracy. For the non-quantized version, the buffer size is set at 128. For the quantized version, we increase the buffer size until its memory use matches the non-quantized versions. Overall the quantized version can store  $[1984~, ~3277~, ~3277~,~ 2726]$ samples at layer 1 to 4 respectively. The accuracy is compared between these in Figure \ref{fig:qt_noise2}. We observe the same trend in terms of impact of delays on the accuracy as in Fig.~\ref{fig:qt_noise}. More importantly we observe for a similar Buffer memory consumption that the quantized version reaches better performances than the non-quantized one in almost all cases, which implies that our quantization modules help improve the accuracy  at a fixed memory budget. In addition, we also have a (highly) favorable bandwidth compression factor (above 15). As a consequence, for the same Buffer Memory budget, and a weaker bandwidth, the quantized version delivers better performances with and without 
slow-downs
(almost everywhere).}




\section{Related work}\label{sec:rel}

To the best of our knowledge \citep{jaderberg2016decoupled} is the the first to directly consider the update or forward locking problems in deep feed-forward networks. Other works \citep{Huo2018,ddg} study the backward locking problem. Furthermore, a number of backpropagation alternatives \citep{choromanska2018beyond,dtargetProp,direct_feedback} 
can address backward locking. However, update locking is a more severe inefficiency. Consider the case where each layer's forward processing time is $T_F$ and is equal across a network of $L$ layers.  
Given that the backward pass is a constant multiple in time of the forward, in the most ideal case the backward unlocking will still only scale as $\mathcal{O}(LT_F)$ with $L$ parallel nodes, while update unlocking could scale as $\mathcal{O}(T_F)$.

One class of alternatives to standard back-propagation aims to avoid its biologically implausible aspects, most notably the weight transport problem \citep{bartunov2018assessing,direct_feedback,feedback_alignment,dtargetProp,ororbia2018conducting,ororbia2019biologically}. Some of these methods \citep{dtargetProp,direct_feedback} can also achieve backward unlocking as they permit all parameters to be updated at the same time, but only once the signal has propagated to the top layer. 
However, they do not solve the update or forward locking problems. Target propagation uses a local auxiliary network as in our approach, for propagating backward optimal activations computed from the layer above.  Feedback alignment replaces the symmetric weights of the backward pass with random weights. Direct feedback alignment extends the idea of feedback alignment passing errors from the top to all layers, potentially enabling simultaneous updates. These approaches have also not been shown to scale to large datasets \citep{bartunov2018assessing}, obtaining only $17.5\%$ top-5 accuracy on ImageNet (reference model achieving $59.8\%$). On the other hand, greedy learning  has been shown to work well on this task \citep{shallow}. We also note concurrent work in the context of biologically plausible models by \citep{nokland2019training} which improves on results from \cite{mostafa2018deep}, showing an approach similar to a specific instantiation of the synchronous version of DGL. This work however does not consider the applications to unlocking nor asynchronous training and cannot currently scale to ImageNet.

Another line of  related work inspired by optimization methods such as \emph{Alternating Direction Method of Multipliers (ADMM)} \citep{admmNet, carreira2014distributed, choromanska2018beyond}  use auxiliary variables to break the optimization into sub-problems. These approaches are fundamentally different from ours as they optimize for the joint training objective, the auxiliary variables providing a link between a layer and its successive layers, whereas we consider a different objective where a layer 
has no dependence on its successors. None of these methods can achieve update or forward unlocking. However, some \citep{choromanska2018beyond} are able to have simultaneous weight updates (backward unlocked). Another issue with ADMM methods is that most of the existing approaches except for \citep{choromanska2018beyond} require standard (``batch'') gradient descent and are thus difficult to scale.
They also often involve an inner minimization problem and have thus not been demonstrated to work on large-scale datasets. Furthermore, none of these have been combined with CNNs.   

Distributed optimization based on data parallelism is a popular area  in machine learning beyond deep learning models and often studied in the convex setting \citep{JMLR:v19:17-650}. For deep network optimization the predominant method
is distributed synchronous SGD \citep{im1hr} and variants, as well as asynchronous  \citep{elasticSGD} variants. Our work is closer to a form of model parallelism rather than data parallelism, and can be easily combined with many data parallel methods (e.g. distributed synchronous SGD). Recently federated learning, where the dataset used on each node is different, has become a popular line of research \cite{konevcny2015federated}. Federated learning is orthogonal to our proposal, but the idea of facilitating different data on each node could potentially be incorporated in our framework. One direction for doing this would consider a class of models where  individual layers can provide both input and output to other layers as done in \cite{huang2016deep}.

Recent proposals for ``pipelining'' \citep{huang2018gpipe} consider systems level approaches to optimize latency times. These methods do not address the update, forward, locking problems\citep{jaderberg2016decoupled} which are algorithmic constraints of the learning objective and backpropagation. Pipelining can be seen as a  attempting to work around these restrictions, with the fundamental limitations remaining. Removing and reducing update, backward, forward locking would simplify the design and efficiency of such systems-level machinery. Tangential to our work \citet{lee2015deeply} considers auxiliary objectives but with a joint learning objective, which is not capable of addressing any of the problems considered in this work.  

We also note that since publication of \cite{belilovsky2020decoupled} several works have extended the methods and studied various things such as how it affects representations \cite{laskin2020parallel} to allow for improved scalability. However, these works largely consider the synchronous setting and have not extensively considered the asynchronous setting emphasized in Sec. 4.3 and 4.4. 

\section{Conclusion}

We have analyzed and introduced a simple and strong baseline for parallelizing per layer and per module computations in CNN training. This work matches or exceeds state-of-the-art approaches addressing these problems and is able to scale to much larger datasets than others. In several realistic settings, for the same bytes range, an asynchronous framework with quantized buffers can be more robust to delays than a non-quantized version while providing better bandwidth. Future work can develop improved auxiliary problem objectives and combinations with delayed feedback.

\subsection*{Acknowledgements}
LL was  supported by ANR-20-CHIA-0022-01 "VISA-DEEP". EO acknowledges NVIDIA for its GPU donation. EB acknowledges funding from IVADO. This work was granted access to the HPC resources of IDRIS under the allocation 2020-[AD011011216R1] made by GENCI and it was partly supported by ANR-19-CHIA "SCAI".

\bibliography{bib_greedy}
\setcounter{theorem}{0}
\newpage
\appendix
\onecolumn

\section{Proofs}
\label{sec:app_proofs}

\setcounter{section}{4}

\begin{lemma}Under Assumption 3 and 4, one has: $\forall \Theta_j, \mathbb{E}_{p^*_{j-1}}\big[\Vert\nabla_{\Theta_j}\ell(Z_{j-1};\Theta_j)\Vert^2\big]\leq G$.\end{lemma}
\begin{proof}
First of all, observe that under Assumption 4 and via Fubini's theorem:
\begin{equation}
    \sum_t c_{j-1}^t = \sum_t \int |p_{j-1}^t(z)-p_{j-1}^*(z)|\,dz=\int\sum_t |p_{j-1}^t(z)-p_{j-1}^*(z)|\,dz <\infty
\end{equation}
thus, $\sum_t |p_j^t-p_j^*|$ is convergent a.s. and $|p_j^t-p_j^*|\rightarrow 0$ a.s as well. From Fatou's lemma, one has:

\begin{align}
    \int p^*_{j-1}(z)\Vert\nabla_{\Theta_j} \ell(z;\Theta_j)\Vert^2 \,dz&= \int \lim\inf_t p^t_{j-1}(z)\Vert\nabla_{\Theta_j} \ell(z;\Theta_j)\Vert^2 \,dz\\ 
    &\leq \lim\inf_t \int  p^t_{j-1}(z)\Vert\nabla_{\Theta_j} \ell(z;\Theta_j)\Vert^2 \,dz\leq G
\end{align}

then, observe that this implies that:
\begin{equation}
    p_j^t(z)\leq p_j^*(z)+(p_j^t(z)-p_j^*(z)) \leq p_j^*(z)+ |p_j^t(z)-p_j^*(z)|\leq p_j^*(z)+\sum_t |p_j^t(z)-p_j^*(z)|
\end{equation}
thus, $\sup_t p_j^t$ is integrable because the right term is integrable as well. Then, observe that:

\[ \Vert \nabla\ell_{j,t} \Vert |p^*_j(z)-p^t_j(z)| =  \Vert \nabla\ell_{j,t} \Vert 1_{p^*_j(z)< p^t_j(z)}|p^*_j(z)-p^t_j(z)|+ \Vert \nabla\ell_{j,t} \Vert 1_{p^*_j(z)\geq  p^t_j(z)}|p^*_j(z)-p^t_j(z)|\]

Then, the left term is bounded because:

\begin{equation}
    \Vert \nabla\ell_{j,t} \Vert 1_{p^*_j(z)< p^t_j(z)}|p^*_j(z)-p^t_j(z)|\leq \Vert \nabla\ell_{j,t} \Vert 1_{p^*_j(z)< p^t_j(z)}p^t_j(z) \leq \Vert \nabla\ell_{j,t} \Vert 1_{p^*_j(z)< p^t_j(z)}\sup_t p^t_j(z)
\end{equation}

\[\int \sum_t   1_{p^*_j(z)\geq  p^t_j(z)}(p^*_j(z)-p^t_j(z))\,dz\leq \int \sum_t   1_{p^*_j(z)=  p^t_j(z)}|p^*_j(z)-p^t_j(z)|\,dz<\infty\]

In particular:
\[\int \sum_t   1_{p^*_j(z)\geq  p^t_j(z)}p^*_j(z)\,dz<\infty\]

It implies that $\sum_t   1_{p^*_j(z)\geq  p^t_j(z)}$ is almost surely finite, and, a.s. $\forall z,\exists t_0, p^*_j(z)\leq p^{t_0}_j(z)$. In particular this implies that a.s.:
\[\forall z, p^*_j(z)\leq \sup_t p_j^t(z)\]

\end{proof}

\begin{lemma}Under Assumptions 1, 3 and 4, we have:
\begin{align*}
\mathbb{E}[\mathcal{L}(\Theta_j^{t+1})]\leq\mathbb{E}[\mathcal{L}(\Theta_j^{t})]-\eta_t\big(\mathbb{E}[\Vert\nabla\mathcal{L}(\Theta_j^t)\Vert^2]-\sqrt{2}Gc^t_{j-1}\big)+\frac{LG}{2}\eta_t^2\,,
\end{align*}
Observe that the expectation is taken over each random variable.
\end{lemma}

\begin{proof}
By $L$-smoothness:
\begin{align}
\mathcal{L}(\Theta_j^{t+1})\leq\mathcal{L}(\Theta_j^t)&+\nabla\mathcal{L}(\Theta_j^t)^{T}(\Theta^{t+1}_{j}-\Theta_{j}^t)+\frac{L}{2}\Vert\Theta_{j}^{t+1}-\Theta_{j}^t\Vert^{2}
\end{align}
Substituting $\Theta_{j}^{t+1}-\Theta_{j}^{t}$ on the right:
\begin{align}
\mathcal{L}(\Theta_j^{t+1})&\leq\mathcal{L}(\Theta_j^t)-\eta_t\nabla\mathcal{L}(\Theta_j^t)^{T} \nabla_{\Theta_j} \ell(Z_{j-1}^t;\Theta_j^t)+\frac{L\eta_t^2}{2}\Vert\nabla_{\Theta_j} \ell(Z_{j-1}^t;\Theta_j^t)\Vert^{2}
\end{align}

Taking the expectation w.r.t. $Z_{j-1}^t$ which has a density $p_{j-1}^t$, we get:
\begin{align*}
\mathbb{E}_{p^t_{j-1}}[\mathcal{L}(\Theta_j^{t+1})]&\leq\mathcal{L}(\Theta_j^t)-\eta_t\nabla\mathcal{L}(\Theta_j^t)^{T}\mathbb{E}_{p_{j-1}^t}[\nabla_{\Theta_j} \ell(Z_{j-1}^t;\Theta_j^t)]
+\frac{L\eta_t^2}{2}\mathbb{E}_{p_{j-1}^t}\big[\Vert\nabla_{\Theta_j} \ell(Z_{j-1}^t;\Theta_j^t)\Vert^{2}\big]
\end{align*}

From Assumption 3, we obtain that:
\begin{equation}
    \frac{L\eta_t^2}{2}\mathbb{E}_{p_{j-1}^t}\big[\Vert \nabla_{\Theta_j} \ell(Z_{j-1}^t;\Theta_j^t)\Vert^{2}\big]\leq \frac{L\eta_t^2G}{2}
\end{equation}

Then, as a side computation, observe that:

\begin{align}
\Vert \mathbb{E}_{p_{j-1}^t}\big[\nabla_{\Theta_j} \ell(Z_{j-1}^t;\Theta_j^t)\big]-\nabla\mathcal{L}(\Theta_j^t)\Vert&=\Vert \int \nabla \ell(z,\Theta_j^t)p^t_{j-1}(z)\,d z -\int \nabla \ell(z,\Theta_j^t)p^*_{j-1}(z)\,d z\Vert\\
&\leq   \int \Vert\nabla \ell(z,\Theta_j^t)\Vert ~ |p^t_{j-1}(z)-p^*_{j-1}(z)|\,d z\\
&=  \int \big(\Vert\nabla \ell(z,\Theta_j^t)\Vert\sqrt{|p^t_{j-1}(z)-p^*_{j-1}(z)|}\big) ~ \sqrt{|p^t_{j-1}(z)-p^*_{j-1}(z)|}\,d z\\
\end{align}
Let us apply the Cauchy-Swchartz inequality, we obtain:

\begin{align}
\Vert \mathbb{E}_{p_{j-1}^t}\big[\nabla_{\Theta_j} \ell(Z_{j-1}^t;\Theta_j^t)\big]-\nabla\mathcal{L}(\Theta_j^t)\Vert&\leq \sqrt{\int \Vert\nabla \ell(z,\Theta_j^t)\Vert^2 |p^t_{j-1}(z)-p^*_{j-1}(z)|\,d z} \sqrt{\int |p^t_{j-1}(z)-p^*_{j-1}(z)|\,d z }\\
&= \sqrt{\int \Vert\nabla \ell(z,\Theta_j^t)\Vert^2 |p^t_{j-1}(z)-p^*_{j-1}(z)|\,d z} \sqrt{c_{j-1}^t}
\end{align}

Then, observe that:

\begin{align}
\int \Vert\nabla \ell(z,\Theta_j^t)\Vert^2 |p^t_{j-1}(z)-p^*_{j-1}(z)|\,d z &\leq  \int \Vert\nabla \ell(z,\Theta_j^t)\Vert^2 \big(p^t_{j-1}(z)+p^*_{j-1}(z)\big)\,d z\\
&=\mathbb{E}_{p^t_{j-1}}[ \Vert\nabla \ell(Z_{j-1},\Theta_j^t)\Vert^2 ]+\mathbb{E}_{p^*_{j-1}}[ \Vert\nabla \ell(Z_{j-1},\Theta_j^t)\Vert^2 ]\\
&\leq 2G
\end{align}

The last inequality follows from Lemma 4.1 and Assumption 3.

Then, using again Cauchy-Schwartz inequality:
\begin{align}
\bigg|\Vert\nabla\mathcal{L}(\Theta_j^t)\Vert^2- \nabla\mathcal{L}(\Theta_j^t)^{T}\mathbb{E}_{p_{j-1}^t}[\nabla_{\Theta_j} \ell(Z_{j-1}^t;\Theta_j^t)]\bigg|
&=\bigg|\nabla\mathcal{L}(\Theta_j^t)^T\big(\nabla\mathcal{L}(\Theta_j^t)-\mathbb{E}_{p_{j-1}^t}[\nabla_{\Theta_j} \ell(Z_{j-1}^t;\Theta_j^t)]\big)\bigg|\\
&\leq \Vert \nabla \mathcal{L}(\Theta^t_j)\Vert~\Vert \mathbb{E}_{p_{j-1}^t}\big[\nabla_{\Theta_j} \ell(Z_{j-1}^t;\Theta_j^t)\big]-\nabla\mathcal{L}(\Theta_j^t)\Vert \\
&\leq \Vert \nabla \mathcal{L}(\Theta^t_j)\Vert \sqrt{2Gc^t_{j-1}}
\end{align}

Then, taking the expectation leads to

\begin{align}
\big|\mathbb{E}\bigg[\Vert\nabla\mathcal{L}(\Theta_j^t)\Vert^2- \nabla\mathcal{L}(\Theta_j^t)^{T}\mathbb{E}_{p_{j-1}^t}[\nabla_{\Theta_j} \ell(Z_{j-1}^t;\Theta_j^t)]\bigg]\big|&\leq \mathbb{E}[\bigg|\Vert\nabla\mathcal{L}(\Theta_j^t)\Vert^2- \nabla\mathcal{L}(\Theta_j^t)^{T}\mathbb{E}_{p_{j-1}^t}[\nabla_{\Theta_j} \ell(Z_{j-1}^t;\Theta_j^t)]\bigg|]\\
&\leq \mathbb{E}[\Vert\nabla \mathcal{L}(\Theta^t_j)\Vert] \sqrt{2Gc^t_{j-1}}\\
&\leq \sqrt{ \mathbb{E}[\Vert\nabla \mathcal{L}(\Theta^t_j)\Vert^2] }\sqrt{2Gc^t_{j-1}}\\
\end{align}
However, observe that by Lemma 4.1 and Jensen inequality:

\begin{align}
\Vert \nabla\mathcal{L}(\Theta_j^t)\Vert^2 =\Vert \mathbb{E}_{p^*_j}[\nabla_{\Theta_j} \ell(Z,\Theta_j^t)]\Vert^2\leq \mathbb{E}_{p^*_j}[ \Vert \nabla_{\Theta_j} \ell(Z,\Theta_j^t)\Vert^2]\leq G
\end{align}
Combining this inequality and Assumption 3, we get:
\begin{align*}
\mathbb{E}[\mathcal{L}(\Theta_j^{t+1})]\leq\mathbb{E}[\mathcal{L}(\Theta_j^{t})]-\eta_t\big(\mathbb{E}[\Vert\nabla\mathcal{L}(\Theta_j^t)\Vert^2]-\sqrt{2}Gc^t_{j-1}\big)+\frac{LG}{2}\eta_t^2\,,
\end{align*}

\end{proof}

\begin{proposition}
Under Assumptions 1, 2, 3 and 4, each term of the following equation converges:
\begin{align}
\hspace{-0.2cm}
\begin{split}
\sum_{t=0}^T \eta_t
\mathbb{E}[\Vert\nabla\mathcal{L}(\Theta_j^t)\Vert^2] \leq  \mathbb{E}[\mathcal{L}(\Theta^0_j)] +G\sum_{t=0}^T\eta_t(\sqrt{2c_{j-1}^t}+\frac{L\eta_t}{2})
\end{split}
\end{align}
\end{proposition}
\begin{proof} Applying Lemma \ref{lemma:main} for $t=0,...,T-1$, we obtain (observe the telescoping sum), for our non-negative loss:
\begin{align}\sum_{t=0}^T \eta_t \mathbb{E}[\Vert\nabla\mathcal{L}(\Theta_j^t)\Vert^2]& \leq  \mathbb{E}[\mathcal{L}(\Theta^0_j)]-\mathbb{E}[\mathcal{L}(\Theta^{T+1}_j)] +\sqrt 2G\sum_{t=0}^T\sqrt{c_j^t}\eta_t +\frac{LG}{2}\sum_{t=0}^T \eta_t^2\\
&\leq \mathbb{E}[\mathcal{L}(\Theta^0_j)] +\sqrt 2G\sum_{t=0}^T\sqrt{c_j^t}\eta_t +\frac{LG}{2}\sum_{t=0}^T \eta_t^2\\
\end{align}

Yet, $\sum \sqrt{c_j^t}\eta_t$ is convergent, as $\sum c_j^t$  and $\sum_t \eta_t^2$ are convergent, thus the right term is bounded.
\end{proof}


\setcounter{section}{1}
\section{Additional Descriptions of Experiments}
Here we provide some additional details of the experiments. Code for experiments is provided along with the supplementary materials.
\paragraph{Comparisons to DNI} The comparison to DNI attempts to directly replicate  the Appendix C.1 \citep{jaderberg2016decoupled}. Although the baseline accuracies for backprop and cDNI are close to those reported in the original work, those of DNI are worse than those reported in \citep{jaderberg2016decoupled}, which could be due to minor differences in the implementation. We utilize a popular pytorch DNI implementation available and source code will be provided.

\paragraph{Auxiliary network study} We use SGD with momentum of $0.9$ and weight decay $5\times 10^{-4}$  \citep{zagoruyko2016wide} and a short schedule of 50 epochs and decay factor of $0.2$ every 15 epochs \citep{shallow}. 
\paragraph{Sequential vs Greedy optimization experiments} We use the same architecture and optimization as in the Auxiliary network study

\paragraph{Imagenet} We use the shortened optimization schedule prescribed in \citep{Xiao2019}. It consists of training for 50 epochs with mini-batch size $256$, uses SGD with momentum of 0.9,  weight decay of $10^{-4}$, and a learning rate of $0.1$ reduced by a factor 10 every 10 epochs.

\section{Detailed Discussion of Relative Speed of Competing Methods}\label{appendix:speed}
Here we describe in more detail the elements governing differences between methods such as DNI\citep{jaderberg2016decoupled}, DDG/FA\citep{Huo2018}, and the simpler DGL. We will argue that if we take the assumption that each approach runs for the same number of epochs or iterations and applies the same splits of the network then DGL is by construction faster than the other methods which rely on feedback.   
The relative speeds of these methods are governed by the following:
\begin{enumerate}
    \item  Computation besides forward and backward passes on primary network modules (e.g. auxiliary networks forward and backward passes) 
    \item Communication time of sending activations from one module to the next module
    \item Communication time of sending feedback to the previous module 
    \item Waiting time for signal to reach final module
\end{enumerate}

As discussed in the text our auxiliary modules which govern (1) for DGL are negligible thus the overhead of (1) is negligible. DNI will inherently have large auxiliary models as it must predict gradients, thus (1) will be much greater than in DGL. (2) should be of equal speed across all methods given the same implementation and hardware. (3) does not exist for the case of DGL but exists for all other cases.  
(4) applies only in the case of backward unlocking methods (DDG/FA) and does not exist for DNI or DGL as they are update unlocked.   

Thus we observe that DGL by construction is faster than the other methods. We note that for use cases in multi-GPU settings communication would need to be well optimized for use of any of these methods. Although we include a parallel implementation based on the software from \citep{ddg}, an optimized distributed implementations of the ideas presented here and related works is outside of the scope of this work.  

\subsection{Auxiliary Network Sizes and FLOP comparisons on ImageNet}
\begin{table}[]
    \centering
    \begin{tabular}{c|c|c}
       & Flops Net & Flops Aux\\\hline
     VGG-13 ($K=4$) &13 GFLOPs & 0.2  GFLOP\\\hline
     VGG-19 ($K=4$) &20 GFLOPs & 0.2 GFLOP\\\hline
     ResNet-152 ($K=2$)  &   11 GFLOP &  0.02 GFLOP\\\hline
    \end{tabular}

    \caption{ImageNet comparisons of FLOPs for auxiliary model in major models trained. Auxiliary networks are negligible.}
    \label{tab:flop_im}
\end{table}
We briefly illustrate the sizes of auxiliary networks. Lets take as an example the ImageNet experiments for VGG-13. At the first layer the output is $224\times224\times64$. The MLP-aux here would be applied after averaging to $2\times2\times 64$, and would consists of 3 fully connected layers of size $256$ ($2*2*64$) followed by a projection to $1000$ image categories. The MLP-SR-aux network used would first reduce to $56\times56\times64$ and then apply 3 layers of $1\times1$ convolutions of width 64. This is followed by reduction to $2\times2$ and 3 FC layers as in the MLP-aux network. As mentioned in Sec. ~\ref{sec:imagenet} the auxiliary networks are neglibile in size. We further illustrate this in \ref{tab:flop_im}.

\section{Additional pseudo-code}\label{appendix:pseudo}
To illustrate the parallel implementations of the Algorithms we show a different pseudocode implementation with an explicit behavior for each worker specified. The following Algorithm~\ref{_algo:basic_parallel} is equivalent to Algorithm~\ref{algo:basic} in terms of output but directly illustrates a parallel implementation. Similarly ~\ref{algo:buff_para} illustrates a parallel implementation of the algorithm described in Algorithm~\ref{algo:buffer_sym}. The probabilities used in Algorithm~\ref{algo:buff_para} are not included here as they are derived from communication and computation speed differences. Finally we illustrate the parallelism compared to backprop in \ref{fig:update_lock_greedy}

    \begin{figure*}[t]
        \centering
        \includegraphics[width=0.6\linewidth]{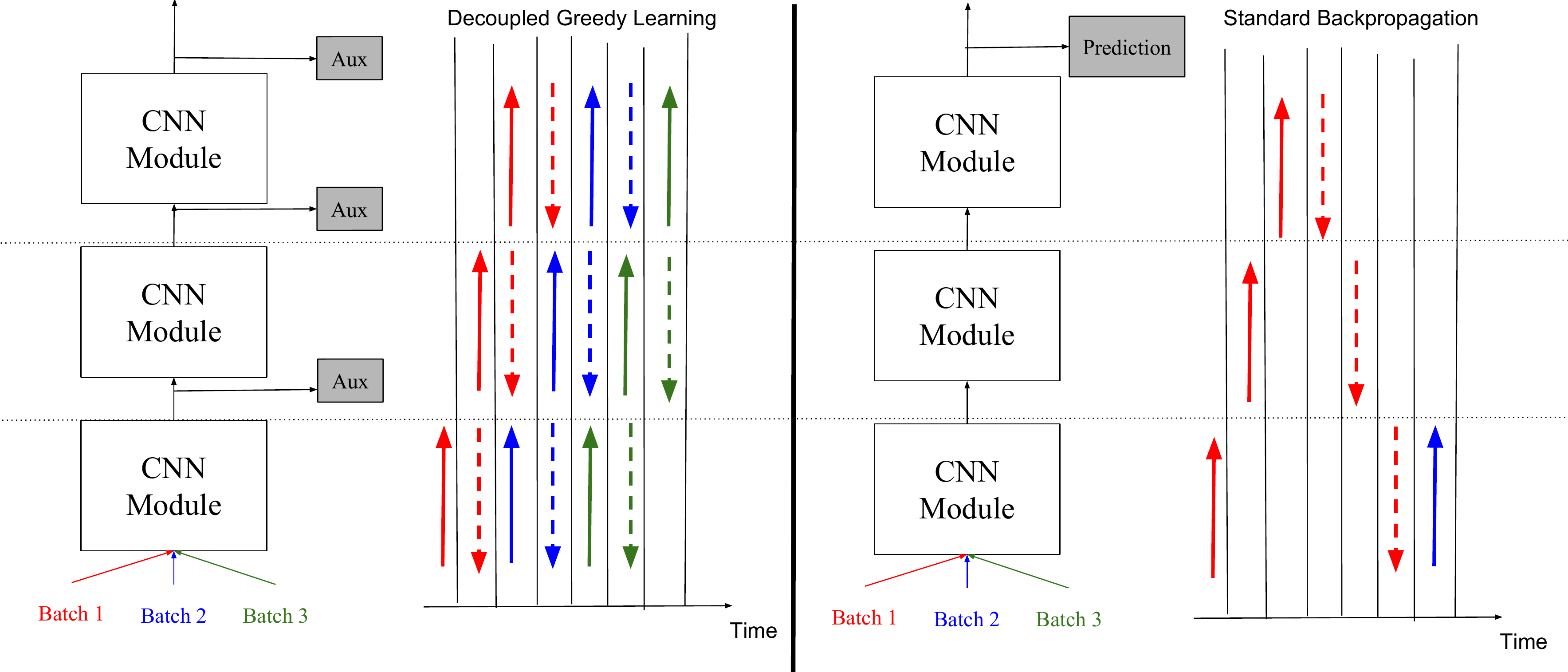}
        \caption{We illustrate the signal propagation for three mini-batches processed by standard back-propagation and with decoupled greedy learning. In each case a module can begin processing forward and backward passes as soon as possible. For illustration we assume same speed for forward and backward passes, and discount the auxiliary network computation (negligible here).} 
    \label{fig:update_lock_greedy}
    \end{figure*}
    
\begin{figure*}[h]
\begin{minipage}{\textwidth}
\begin{minipage}[b]{0.51\textwidth}
\begin{algorithm2e}[H]\small
\caption{DGL Parallel Implementation\label{_algo:basic_parallel}}
  \SetAlgoLined
  \DontPrintSemicolon
 \KwIn{Stream $\mathcal{S}\triangleq\{(x_0^t,y^t)\}_{t\leq T}$ of samples or mini-batches;}
\textbf{Initialize} Parameters $\{\theta_j,\gamma_j\}_{j\leq J}$\\
Worker 0:\;
\For {$x_0^t \in \mathcal{S}$}{
$x^t_1 \leftarrow f_{\theta^t_{0}}(x^t_{0})$\\
Send $x_0^t$ to worker $1$\;
Compute $\nabla_{(\gamma_1,\theta_1)}\hat{ \mathcal{L}}(y^t,x^t_0;\gamma^t_0,\theta^t_0)$\;
$(\theta^{t+1}_0,\gamma^{t+1}_0)\leftarrow$ Step of parameters $(\theta^t_0,\gamma^t_0)$}
Worker $j$:\;
\For{$t \in 0 ... T$}{
Wait until $x^t_{j-1}$ is available\;
$x^t_j \leftarrow f_{\theta^t_{j-1}}(x^t_{j-1})$ \\
Compute $\nabla_{(\gamma_j,\theta_j)}\hat{ \mathcal{L}}(y^t,x^t_j;\gamma^t_j,\theta^t_j)$\\
Send $x^t_{j}$ to worker $x^t_{j+1}$\\
$(\theta^{t+1}_j,\gamma^{t+1}_j)\leftarrow$ Step of parameters $(\theta^t_j,\gamma^t_j)$}
\end{algorithm2e}
\end{minipage}
\hfill
\begin{minipage}[b]{0.46\textwidth}
\begin{algorithm2e}[H]
\small\caption{DGL Async Buffer Parallel Impl.}\label{algo:buff_para}
\SetAlgoLined
  \DontPrintSemicolon
    \KwIn{Stream $\mathcal{S}\triangleq\{(x_0^t,y^t)\}_{t\leq T}$;  Distribution of the delay $p=\{p_j\}_{j}$; Buffer size $M$ }
 \textbf{Initialize:} Buffers $\{B_j\}_{j}$ with size $M$; params $\{\theta_j,\gamma_j\}_{j}$\\
 Worker $j$: \\
\While{\normalfont{\textbf{ training}}}{
    \lIf{ $ j=1$}{ $ (x_{0},y)\gets \mathcal{S}$} \lElse{ $(x_{j-1},y)\gets B_{j-1}$}\;
    $x_j \leftarrow f_{\theta_{j-1}}(x_{j-1})$\;
    Compute $\nabla_{(\gamma_j,\theta_j)}\hat{ \mathcal{L}}(y,x_j;\gamma_j,\theta_j)$\;
     $(\theta_j,\gamma_j)\leftarrow$ Step of parameters $(\theta_j,\gamma_j)$\;
    \lIf{$j<J$}{
    $B_{j} \gets (x^{j},y)$
           }}
\end{algorithm2e}
\end{minipage}
\end{minipage}
\end{figure*}


\end{document}